\documentclass{article}
\usepackage{times}

\usepackage{graphicx} % more modern
\usepackage{subfigure} 
\usepackage{balance}
\usepackage{natbib}
\usepackage{algorithm}
\usepackage{algorithmic}
\usepackage{array}
\usepackage{amsmath}
\usepackage{amssymb}
\usepackage{amsthm}
\usepackage[textwidth=1.2cm]{todonotes}	
\newtheorem{lemma}{Lemma}
\newtheorem{theorem}{Theorem}

\newtheorem{corollary}{Corollary}
\newtheorem{assumption}{Assumption}
\newtheorem{remark}{Remark}

\newcommand{\cN}{\mathcal{N}}
\newcommand{\cO}{\mathcal{O}}
 \newcommand{\cL}{\mathcal{L}}
\newcommand{\EE}[1]{\mathbb{E}\left[#1\right]}
\newcommand{\E}{\mathbb{E}}

\DeclareMathOperator*{\arginf}{arg\,inf}
\DeclareMathOperator{\polylog}{polylog}
\DeclareMathOperator{\logloglog}{logloglog}
\DeclareMathOperator{\polyloglog}{polyloglog}
\newcommand*{\MyDef}{\mathrm{def}}
\newcommand*{\eqdefU}{\ensuremath{\mathop{\overset{\MyDef}{=}}}}
\newcommand*{\eqdef}{\mathop{\overset{\MyDef}{\resizebox{\widthof{\eqdefU}}{
\heightof{=}}{=}}}}

\definecolor{babyblue}{rgb}{0.54, 0.81, 0.94}
\definecolor{citrine}{rgb}{0.89, 0.82, 0.04}

\usepackage{hyperref}

\usepackage[accepted]{icml2015} 
\icmltitlerunning{Simple regret for infinitely many armed bandits}

\begin{document} 

\twocolumn[
\icmltitle{Simple regret for infinitely many armed bandits}

% unless you've provided the [accepted] option to the icml2013
\vspace*{-0.5em}
\icmlauthor{Alexandra Carpentier}{a.carpentier@statslab.cam.ac.uk}
%\vspace*{-0.1em}
\icmladdress{Statistical Laboratory, CMS, Wilberforce Road, CB3 0WB, University of Cambridge, United Kingdom}
\icmlauthor{Michal Valko}{michal.valko@inria.fr%\hspace*{-0.3em}
}
%\vspace*{-0.1em}
 \icmladdress{INRIA Lille - Nord Europe, SequeL team, 40 avenue Halley 59650, Villeneuve d'Ascq, France}
\vspace*{-0.2em}

\icmlkeywords{Bandit Theory, Online Learning}

\vskip 0.3in
%\vskip -0.05in
]

\begin{abstract}
We consider a stochastic bandit problem with infinitely many arms. In this setting, the learner has no chance of trying all the arms even once and has to dedicate its limited number of samples only to a certain number of arms. All previous algorithms for this setting were designed for minimizing the cumulative regret of the learner. In this paper, we propose an algorithm aiming at minimizing the simple regret. As in the cumulative regret setting of infinitely many armed bandits, the rate of the simple regret will depend on a parameter $\beta$ characterizing the distribution of the near-optimal arms. We prove that depending on $\beta$, our algorithm is minimax optimal either up to a multiplicative constant or up to a $\log(n)$ factor. We also provide extensions to several important cases: when $\beta$ is unknown, in a natural setting where the near-optimal arms have a small variance, and in the case of unknown time horizon.
%we provide a matching lower bound which shows that our algorithm is minimax optimal.
\end{abstract}
%\vskip  -1em
\section{Introduction}
%\vspace{-0.1em}
%In the beginnings of sequential decision making, most of the research was concern with 
%how to optimally select from small number of actions. However, there are many situations
%where we are faced with a large number of possible actions. Whenever the actions 
%exhibit some useful structural property, one can attempt to use it to deal with a big action space.
%This has lead to linear, convex, combinatorial, submodular, contextual and other settings. 
%Typically the results scale with number of possible actions or some dimension of the space.

Sequential decision making has been recently fueled by several industrial applications, e.g.,~advertisement, and recommendation systems. In many of these situations, the learner is faced with a large number of possible actions, among which it has to make a decision. The setting we consider is a direct extension of a classical decision-making setting, in which we only receive feedback for the actions we choose, the \textit{bandit setting}. In this setting, at each time~$t$, the learner can choose among all the actions (called the \emph{arms}) and receives a sample (\emph{reward}) from the chosen action, which is typically a noisy characterization of the action. The learner performs $n$ such rounds and its performance is then evaluated with respect to some criterion, for instance the cumulative regret or the simple regret. 
% as e.g.~maximizing the sum of all rewards (cumulative regret criterion).

In the classical, multi-armed bandit setting, the number of actions is assumed to be finite and small when compared to the number of decisions. In this paper, we consider an extension of this setting to infinitely many actions, the \textit{infinitely many armed bandits}~\cite{berry1997bandits,wang2008algorithms,bonald2013two-target}. Inevitably, the sheer amount of possible actions makes it impossible to try each of them even once. Such a setting is practically relevant for cases where one faces a finite, but extremely large number of actions. This setting was first formalized by~\citet{berry1997bandits} as follows. At each time $t$, the learner can either sample an arm (a distribution) that has been already observed in the past, or sample a new arm, whose mean $\mu$ is sampled from the \textit{mean reservoir distribution} $\cL$.

The additional challenges of the infinitely many armed bandits with respect to the multi-armed bandits come from two sources. 
First, we need to find a good arm among the sampled ones. Second, we need to sample 
(at least once) enough arms in order to have (at least once) a reasonably good one. These two difficulties ask for a while which we call the \textit{arm selection tradeoff}. It is different from the known \textit{exploration/exploitation tradeoff} and more linked to model selection principles: On one hand, we want to sample only from a small subsample of arms so that we can decide, with enough accuracy, which one is the best one among them. On the other hand, we want to sample as many arms as possible in order to have a higher chance to sample a good arm at least once. This tradeoff makes the problem of infinitely many armed bandits significantly different from the classical bandit problem.

\citet{berry1997bandits} provide asymptotic, minimax-optimal (up to a $\log n$ factor) bounds for the \textit{average cumulative regret}, defined as the difference between  $n$ times the highest possible value $\bar \mu^*$ of the mean reservoir distribution and the mean of the sum of all samples that the learner collects. A follow-up on this result was the work of~\citet{wang2008algorithms}, providing algorithms with finite-time regret bounds and the work of~\citet{bonald2013two-target}, giving an algorithm that is optimal with exact constants in a strictly more specific setting. In all of this prior work, the authors show that it is the \textit{shape} of the arm reservoir distribution what characterizes the \textit{minimax-optimal rate} of the average cumulative regret. Specifically, \citet{berry1997bandits} and~\citet{wang2008algorithms} assume that the mean reservoir distribution is such that, for a small $\varepsilon>0$, locally around the best arm $\bar \mu^*$, we have that
\begin{align}\label{eq:refreg}
\mathbb P_{\mu \sim \cL}\left(\bar \mu^* - \mu \geq \varepsilon\right)\approx \varepsilon^\beta,
\end{align}
that is, they assume that the mean reservoir distribution is $\beta$-regularly varying in $\bar \mu^*$. 
When this assumption is satisfied with a known $\beta$, their algorithms achieve an expected cumulative regret of order
\begin{equation}\label{eq:cumreg}
\!\!\!\!\!\EE{R_n}\!=\!\cO\left(\max\left(n^{\frac{\beta}{\beta+1}} \polylog n, \sqrt{n} \polylog n\right) \right).
\end{equation}
The limiting factor in the general setting is a $1/\sqrt{n}$ rate for estimating the mean of any of the arms with $n$ samples. This gives the rate~\eqref{eq:cumreg} of~$\sqrt{n}$. It can be refined if the distributions of the arms, that are sampled from the mean reservoir distribution, 
are Bernoulli of mean~$\mu$ and $\bar \mu^* = 1$  or in the same spirit, if the distributions of the arms are defined on $[0,1]$ and $\bar \mu^* = 1$~as
\begin{equation}\label{eq:cumreg2}
\EE{R_n}  = \cO\left(n^{\frac{\beta}{\beta+1}} \polylog n\right).
\end{equation}
\citet{bonald2013two-target} refine the result~\eqref{eq:cumreg2} even more  by removing the $\polylog n$ factor and proving upper and lower bounds that \textit{exactly match}, even in terms of constants, for a specific sub-case of a uniform mean reservoir distribution. Notice that the rate~\eqref{eq:cumreg2} is faster than the more general rate~\eqref{eq:cumreg}. This comes from the fact that they  assume that the variances of the arms decay with their quality, making finding a good arm easier. For both rates (\ref{eq:cumreg} and~\ref{eq:cumreg2}), $\beta$~is the \textit{key parameter} for solving the arm selection tradeoff: with smaller $\beta$ it is more likely that the mean reservoir distribution outputs a high value, and therefore, we need fewer arms for the optimal arm selection tradeoff.

Previous algorithms for this setting were designed for minimizing the cumulative regret of the learner which optimizes the cumulative sum of the rewards. In this paper, we consider the problem of minimizing the \textit{simple regret}. We want to select an optimal arm given the time horizon $n$. The \textit{simple regret} is the difference between the mean of the arm that the learner selects at time $n$ and the highest possible mean $\bar \mu^*$. The problem of minimizing the simple regret in a multi-armed bandit setting (with finitely many arms) has recently attracted significant attention \citep{even2006action, audibert2010best, kalyanakrishnan2012pac, kaufmann2013information, karnin2013almost, gabillon2012best, jamieson2014lilUCB} and algorithms have been developed either in the setting of a fixed budget which aims at finding an optimal arm or in the setting of a \emph{floating} budget which aims at finding an $\varepsilon$-optimal arm.

All prior work on simple regret considers a fixed number of arms that will be ultimately all explored and cannot be applied to an infinitely many armed bandits or to a bandit problem with the number of arms larger than the available time budget. An example where efficient strategies for minimizing the simple regret of an infinitely many armed bandit are relevant is the search of a good \textit{biomarker} in biology, a single \textit{feature} that performs best on average~\cite{hauskrecht2006fundamentals}. There can be too many possibilities that we cannot afford to even try each of them in a reasonable time. Our setting is then relevant for this special case of \textit{single feature selection}. In this paper, we provide the following results for the simple regret of an infinitely many armed bandit, a problem that was  not considered before.
%of \textit{simple regret} in the infinitely many armed bandit %that add to both the state of the art for infinitely many armed and cumulative regret~\cite{berry1997bandits,wang2008algorithms,bonald2013two-target} and also differ from the results on multi-armed bandit optimizing simple regret~\cite{audibert2010best}.
\begin{itemize}
\item We propose an algorithm that for a fixed horizon~$n$ achieves the finite-time simple regret rate
$$r_n = \cO\left(\max\left(n^{-1/2}, n^{-\frac{1}{\beta}} \polylog n\right)\right).$$
\item We prove corresponding lower bounds for this infinitely many armed simple regret problem, that are matching up to a multiplicative constant for $\beta<2$, and matching up to a $\polylog n$ for $\beta\geq 2$.% and up to a $\polylog n$ for $\beta= 2$.
\item We provide three important extensions:
\begin{itemize}
\item The first extension concerns the case where the distributions of the arms are defined on $[0,1]$ and where~$\bar \mu^* = 1$. In this case, replacing the Hoeffding bound in the confidence term of our algorithm by a Bernstein bound, bounds the simple regret as
\begin{align*}
\hspace{-0.7cm} r_n\!=\!\cO\big(\!\max(\tfrac{1}{n} \polylog n, (n \log n)^{-\frac{1}{\beta}} \polyloglog n \big).
\end{align*}
\item The second extension treats \textit{unknown} $\beta$. We prove that it is possible to estimate $\beta$ with enough precision, so that its knowledge is not necessary for implementing the algorithm. This can be also applied to the prior work~\cite{berry1997bandits,wang2008algorithms} where $\beta$ is also necessary for implementation and optimal bounds.
\item Finally, in the third extension we make the algorithm anytime using known tools.
\end{itemize}
\item We provide simple numerical simulations of our algorithm and compare
it to infinitely many armed bandit algorithms optimizing  cumulative regret and to  multi-armed bandit
algorithms optimizing simple regret.
\end{itemize}

%Again, the rate of our algorithm will be characterized by $\beta$

%In particular, we consider we have a budget of 
%$n$ pulls where we are free to use for the exploration after which we have a to recommend the best arm. The performance 
%of the learner depends solely on this final arm. 

%*** talk about $\beta$-Weibull regularity

%*** 
%The algorithms of \cite{bonald2013two-target} is not optimal in terms of simple regret.  Neither is ours, if we restrict ourselves to the Bernoulli distribution. 
%However, we can easily modify our algorithm by using Bernstein bound instead the of by Hoeffding to get a minimax algorithm
%for the Bernoulli distribution  as well.

Besides research on infinitely many arms bandits, there exist many other settings 
where the number of actions may be infinite. One class of examples is fixed design 
such as linear bandits~\cite{dani2008stochastic} other settings
consider bandits in known or unknown metric space~\cite{kleinberg2008multi,munos2014from,azar2014online}.
These settings assume regularity properties that are very different from the properties 
assumed in the infinitely many arm bandits and give rise to significantly different approaches and results. Furthermore, in classic optimization settings, one assumes that in addition to the rewards, there is side information available through the position of the arms, combined with a smoothness assumption on the reward, which is much more restrictive. On the contrary, we only assume a bound on the proportion of near-optimal arms. It is not always the case that there is side information through a topology on the arms. In such cases, the infinitely many armed setting is applicable while optimization routines are not.

%Furthermore, if there is a topology but it is high dimensional, the smoothness assumptions is less relevant, and it becomes less and less possible to take advantage of it.
\section{Setting}

\paragraph{Learning setting}

Let $\tilde {\mathcal L}$ be a distribution of distributions. We call $\tilde{\mathcal L}$ the \emph{arm reservoir distribution}, i.e.,~the distribution of the means of arms. Let $\mathcal L$ be the distribution of the means of the distributions output by $\tilde {\mathcal L}$, i.e.,~the \textit{mean reservoir distribution}.
% Let $F$ be the associated distribution function. 
%Let $F^{-1}$ the pseudo-inverse of the arm reservoir distribution.
Let $\mathbb A_t$ denote the changing set of $K_t$ arms at time $t$.

 %Let $\mu \sim \tilde \nu$ be a sample from this distribution, and let $\nu_{t,\mu}$ be a distribution of mean $\mu$ that corresponds 

At each time $t+1$, the learner can either choose an arm $k_{t+1}$ among the set of the $K_t$ arms $\mathbb A_t = \{\nu_1, \ldots, \nu_{K_t}\}$ that it has already observed (in this case, $K_{t+1} = K_t$ and $\mathbb A_{t+1} = \mathbb A_t$), or choose to get a sample of a new arm that is generated according to $\tilde {\mathcal L}$ (in this case, $K_{t+1} = K_t+1$ and $\mathbb A_{t+1} = \mathbb A_t \cup \{\nu_{K_t +1}\}$ where $\nu_{K_t +1} \sim \tilde {\mathcal L}$). Let $\mu_i$ be the mean of arm $i$, i.e.,~the mean of distribution $\nu_i$ for $i \leq K_t$. We assume that $\mu_i$ always exists.

In this setting, the learner observes  a sample at each time. At the end of the horizon, which happens at a given time $n$, the learner has to output an arm $\widehat k \leq K_n$, and its performance is assessed by the simple regret$$r_n = \bar \mu^* - \mu_{\widehat k},$$
where $\bar \mu^* = \arginf_{m} \left(\mathbb P_{\mu \sim \cL} (\mu \leq m) = 1\right)$ is the right end point of the domain.% and we assume that this domain is bounded (and that $\bar \mu^*< \infty$).

\paragraph{Assumption on the samples}
The domain of the arm reservoir distribution $\tilde {\mathcal L}$ are distributions of arm samples. We assume that these distributions $\nu$ are bounded.

\begin{assumption}[Bounded distributions in the domain of $\tilde {\mathcal L}$]\label{ass:sample}
Let $\nu$ be a distribution in the domain of $\tilde {\mathcal L}$. Then $\nu$ is a bounded distribution. Specifically, there exists an universal constant $C>0$ such that the domain of $\nu$ is contained in $[-C,C]$.
\end{assumption}
This implies that the expectations of all distributions generated by $\tilde {\mathcal L}$ exist, are finite, and bounded by $C$. In particular, this implies that
$$\bar \mu^* = \arginf_{m} \left(\mathbb P_{\mu \sim \mathcal L} (\mu \leq m) = 1\right)< +\infty,$$
which implies that the regret is well defined, and that the domain of $\mathcal L$ is bounded by $2C$. Note that all the results that we prove hold also for sub-Gaussian distributions $\nu$ and bounded $\mathcal L$. Furthermore, it would possible to relax the sub-Gaussianity using different estimators recently developed for heavy-tailed distributions~\citep{catoni2012challenging}.

\paragraph{Assumption on the arm reservoir distribution}

%We first assume that the domains of the distributions of the arm reservoir distribution, i.e.,~the distributions of the mean of the arms, are bounded. 
%\begin{assumption}[Bounded domain]\label{ass:domain}
%We assume that the domain of $\mathcal L$ is lower bounded and that $\bar \mu^*$ is the supremum of the domain, i.e.
%$$\bar \mu^* = \arginf_{m} \left(\mathbb P_{\mu \sim \mathcal L} (\mu \leq m) = 1\right)< +\infty.$$
%We assume that $M$ bounds the range of the domain. 
%\end{assumption}

We now assume that the mean reservoir distribution $\mathcal L$ has a certain regularity in its right end point, which is a standard assumption for infinitely many armed bandits. Note that this implies that the distribution of the means of the arms is in the domain of attraction of a Weibull distribution, and that it is related to assuming that the distribution is $\beta$ regularly varying in its end point $\bar \mu^*$.
\begin{assumption}[$\beta$ regularity in $\bar \mu^*$]\label{ass:reg}
Let $\beta>0$. There exist $\tilde E,\tilde E'>0$, and $0< \tilde B<1$ such that for any $0\leq \varepsilon \leq \tilde B$,
$$\tilde E' \varepsilon^{\beta} \geq \mathbb P_{\mu \sim \mathcal L}\left(\mu > \bar \mu^* - \varepsilon\right)\geq \tilde E \varepsilon^{\beta}.$$
\end{assumption}
This assumption is the same as the classical one~\eqref{eq:refreg}. Standard bounded distributions satisfy Assumption~\ref{ass:reg} for a specific $\beta$, e.g.,~all the $\beta$ distributions, in particular the uniform distribution, etc.
% Note t
%, and we rewrite it in this way for convenience.

\section{Main results} %: lower bounds, upper bounds with the SiRI algorithm and extensions}

In this section, we first present the information theoretic lower bounds for the infinitely many armed bandits with simple regret as the objective. 
We then present our algorithm and its analysis proving the upper bounds that match the lower bounds --- in some cases, depending on $\beta$, up to a $\polylog n$ factor. This makes our algorithm (almost) \emph{minimax} optimal. Finally, we provide three important extensions as corollaries.
%Moreover, we also provide lower bounds for this problem to show that our algorithm is minimax optimal.

\subsection{Lower bounds}

The following theorem exhibits the \emph{information theoretic complexity} of our problem
and is proved in 
%the full paper~\cite{carpentier2014asimple}.
 Appendix~\ref{proof:lb}. 
 Note that the rates crucially depend on $\beta$.
\begin{theorem}[Lower bounds]\label{thm:lb}

Let us write $\mathcal S_{\beta}$ for the set of distributions of arms distributions $\tilde {\mathcal L}$ that satisfy Assumptions~\ref{ass:sample} and~\ref{ass:reg} for the parameters~$\beta, \tilde E, \tilde E', C$. Assume that $n$ is larger than a constant that depends on $\beta, \tilde E, \tilde E',\tilde B, C$. Depending on the value of~$\beta$, we have the following results, for any algorithm $\mathcal{A}$, where $v$ is a small enough constant.
\begin{itemize}
\item Case $\beta<2$: With probability larger than $1/3$,
\begin{align*}
\inf_{\mathcal{A}} \sup_{\tilde {\mathcal L} \in \mathcal S_{\beta}} r_n &\geq v n^{-1/2}.
\end{align*}

\item Case $\beta \geq 2$: With probability larger than $1/3$,
\begin{align*}
\inf_{\mathcal{A}}\sup_{\tilde {\mathcal L} \in \mathcal S_{\beta}} r_n \geq v n^{-1/\beta}.
\end{align*}

\end{itemize}

\end{theorem}

\begin{remark}\label{rem:refim}
\normalfont
Comparing these results with the rates for the cumulative regret problem~\eqref{eq:cumreg} from the prior work, one can notice
that there are two regimes for the cumulative regret results. One regime is 
characterized by a rate of $\sqrt{n}$ for $\beta \leq 1$, 
and the other characterized by a $n^{\beta/(1+\beta)}$ rate for $\beta \geq 1$. Both of these regimes are related to the arm selection tradeoff. The first regime corresponds to \textit{easy} problems where the mean reservoir distribution puts a high mass close to $\bar \mu^*$, which favors sampling a good arm with high mean from the reservoir. In this regime, the $\sqrt{n}$ rate comes from the parametric $1/\sqrt{n}$ rate for estimating the mean of any arm with $n$ samples. The second regime corresponds to \textit{more difficult} problems where the reservoir is unlikely to output a distribution with mean close to $\bar \mu^*$ and where one has to sample many arms from the reservoir. In this case, the $\sqrt{n}$ rate is not reachable anymore because there are too many arms to choose from sub-samples of arms containing good arms.
The same dynamics exists also for the \textit{simple regret}, where 
there are again two regimes, one characterized by a 
$n^{-1/2}$ rate for $\beta \leq 2$, and the other characterized 
by a $n^{-1/\beta}$ rate for $\beta \geq 2$. 
Provided that these bounds are tight (which is the case, up to a $\polylog n$, Section~\ref{ss:uppersiri}), one can see that there is an 
interesting  difference between the cumulative regret problem 
and the simple regret one. Indeed, the change of regime 
is here for $\beta = 2$ and not for $\beta = 1$, i.e.,~the 
parametric rate of $n^{-1/2}$ is valid for larger values of 
$\beta$ for the simple regret. This comes from the fact that 
for the simple regret objective, there is no exploitation 
phase and everything is about exploring. Therefore, an 
optimal strategy can spend more time exploring the set of 
arms and reach the parametric rate also in situations where 
the cumulative regret does not correspond to the parametric 
rate. This has also practical implications examined empirically in Section~\ref{sec:exp}.

\end{remark}

\subsection{SiRI and its upper bounds}
\label{ss:uppersiri}

In this section, we present our algorithm, the Simple Regret for Infinitely many arms (SiRI) and its analysis.

\paragraph{The SiRI algorithm}

Let
$b = \min(\beta,2),$
and let
$$\bar T_{\beta} = \lceil A(n) n^{b/2} \rceil,$$
where
 $$
  A(n)=\begin{cases}
    A, & \text{if $\beta<2$}\\
		A/\log(n)^2, & \text{if $\beta=2$}\\
		A/\log(n), & \text{if $\beta>2$}    
  \end{cases}
$$

where $A$ is a small constant whose precise value will depend on our analysis.
Let $\log_2$ be the logarithm in base 2. Let us define
$$\bar t_{\beta} = \lfloor \log_2(\bar T_{\beta}) \rfloor.$$

Let $T_{k,t}$ be the number of pulls of arm $k \leq K_t$, and $X_{k,u}$ for the $u$-th sample of $\nu_k$. 
The empirical mean of the samples of arm $k$ is defined as
$$\widehat \mu_{k,t} = \frac{1}{T_{k,t}} \sum_{u=1}^{T_{k,t}} X_{k,u}.$$
With this notation, we provide SiRI as Algorithm~\ref{alg:siri}.

\begin{algorithm}[t]
\begin{algorithmic}
\STATE {\bf Parameters:} $\beta, C, \delta$
\STATE \textbf{Initial pull of arms from the reservoir:} 
\STATE Choose $\bar T_{\beta}$ arms from the reservoir $\tilde {\mathcal L}$ .
\STATE Pull each of $\bar T_{\beta}$ arms once.
\STATE $t \gets \bar T_{\beta}$

\STATE \textbf{Choice between these arms:}

\WHILE{$t \leq n$}
\STATE For any $k \leq \bar T_{\beta}$: 
\begin{align}
B_{k,t} \gets  \widehat \mu_{k,t} &+  2\sqrt{\frac{C}{T_{k,t}}\log\big(2^{2\bar t_{\beta}/b}/(T_{k,t}\delta)\big)}\nonumber\\ &+ 
\frac{2C}{T_{k,t}}\log\left(2^{2\bar t_{\beta}/b}/(T_{k,t}\delta)\right)\label{eq:UCBSiRI}
\end{align}
\STATE Pull $T_{k,t}$ times the arm $k_t$ that maximizes $B_{k,t}$ and \quad receive $T_{k,t}$ samples from it.
\STATE $t\leftarrow t+ T_{k,t}$
\ENDWHILE
\STATE \textbf{Output:} Return the most pulled arm $\widehat k$.
\end{algorithmic}
\caption{SiRI \\ \hspace*{0.5cm}
\textit{Simple Regret for Infinitely Many Armed Bandits}
\label{alg:siri}}
\end{algorithm}

\paragraph{Discussion}
% \normalfont
SiRI is a UCB-based algorithm, where the leading confidence term is of order
\[ \sqrt{\frac{\log\left(n/(\delta T_{k,t})\right)}{T_{k,t}}}\cdot \]
Similar to the MOSS algorithm~\cite{audibert2009minimax}, we divide the $\log(\cdot)$ term by $T_{k,t}$, in order to avoid additional logarithmic factors in the bound. But a simpler algorithm with a confidence term as in a classic UCB algorithm for cumulative regret,
\[\sqrt{\frac{\log(n/\delta)}{T_{k,t}}}\text{,}\]
would provide almost optimal regret, up to a $\log n$, i.e.,~with a slightly worse regret than what we get. It is quite interesting that with such a confidence term, SiRI  is optimal for minimizing the \textit{simple} regret for infinitely many armed bandits, since MOSS, as well as the classic UCB algorithm, targets the cumulative regret. The main difference between our strategy and the cumulative strategies~\cite{berry1997bandits,wang2008algorithms,bonald2013two-target} is in the number of arms sampled from the arm reservoir:  For the simple regret, we need to sample more arms. Although the algorithms are related, their analyses are quite different: Our proof is \textit{event-based} whereas the proof for the cumulative regret targets \textit{directly the expectations}.

It is also interesting to compare SiRI with existing algorithms targeting the simple regret for finitely many arms, as the ones by~\citet{audibert2010best}. SiRI can be related to their UCB-E with a specific confidence term  and a specific choice of the number of arms selected. Consequently, the two algorithms are related but the regret bounds obtained for UCB-E are not informative when there are infinitely many arms. Indeed, the theoretical performance of UCB-E is decreasing with the sum of the inverse of the gaps squared, which is infinite when there are infinitely many arms. In order to obtain a useful bound in this case, we need to consider a more refined analysis which is the one that leads to Theorem~\ref{thm:ub}.% Moreover, UCB-E needs, in order to be optimal, to be calibrated using the sum of the inverse of the gaps squared of the sampled arms, which is an unknown, problem dependent quantity.
% \end{remark}

\begin{remark}
\normalfont
Note that SiRI pulls series of samples from the same arm without updating the estimate which may seem wasteful.
In fact, it is possible to update the estimates after each pull.  
On the other hand, SiRI is already minimax optimal, so one can only hope to get  improvement in constants. 
Therefore, we present this version of SiRI, since its analysis is easier to follow. 
\end{remark}

%\subsection{Main result}

\paragraph{Main result}
We now state the main result which characterizes SiRI's simple regret according to $\beta$.

\begin{theorem}[Upper bounds]\label{thm:ub}
%\label{thm:upper}
Let $\delta>0$. Assume all Assumptions~\ref{ass:sample} and~\ref{ass:reg} of the model and that $n$ is larger than a large constant that depends on $\beta, \tilde E, \tilde E', \tilde B, C$. Depending on the value of $\beta$, we have the following results, where $E$ is a large enough constant.
\begin{itemize}
\item Case $\beta<2$: With probability larger than $1-\delta$,
\begin{align*}
r_n &\leq E n^{-1/2}  \log(1/\delta) (\log(\log(1/\delta)))^{96} \sim n^{-1/2}.
\end{align*}

\item Case $\beta> 2$: With probability larger than $1-\delta$,
\begin{align*}
r_n &\leq E (n\log(n))^{-1/\beta}  (\log(\log(\log(n)/\delta)))^{96} \times\\
 & \times\log(\log(n)/\delta) \sim (n\log n)^{-1/\beta} \polyloglog n.
\end{align*}

\item Case $\beta = 2$: With probability larger than $1-\delta$,
\begin{align*}
r_n &\leq E \log(n) n^{-1/2}  (\log(\log(\log(n)/\delta)))^{96}\times\\
&\times\log(\log(n)/\delta) \sim n^{-1/2} \log n  \polyloglog n.
\end{align*}
\end{itemize}
\end{theorem}
\begin{proof}[Short proof sketch]
In order to prove the results, the main tools are events $\xi_1$ and $\xi_2$
%~\cite{carpentier2014asimple}. 
(Appendix~\ref{ub}). 
One event controls the number of arms at a given distance from $\bar \mu^*$ and the other one controls the distance between the empirical means and the true means of the arms. 

Provided that events  $\xi_1$ and $\xi_2$ hold, which they do with high probability, we know that there are less than approximately $N_u = \bar T_{\beta} 2^{-u}$ arms at a distance larger than $2^{-u/\beta}$ from $\bar \mu^*$, and that each arm that is at a distance larger than $2^{-u/\beta}$ from $\bar \mu^*$ will be pulled less than $P_u = 2^{2u/\beta}$ times. After these many pulls, the algorithm recognizes that it is suboptimal.

Since a simple computation yields% rate of the simple regret is then of order $r_n = \frac{1}{T_{\beta}^{1/\beta}}$, since %where $U^*$ is the largest $U^*$ such that
$$\sum_{0 \leq u \leq \log_2(\bar T_{\beta})} N_u P_u  \leq \frac{n}{C},$$
we know that all the suboptimal arms at a distance further than $2^{-\log_2(\bar T_{\beta})/\beta}$ from the optimal arm are discarded since they are all sampled enough to be proved suboptimal. We thus know that an arm at a distance less than $2^{-\log_2(T_{\bar \beta})/\beta}$ from the optimal arm is selected in high probability, which concludes the proof.

The full proof%
%~\cite{carpentier2014asimple} 
(Appendix~\ref{ub}) 
is quite technical, since it uses a peeling argument to correctly define the high probability event to avoid a suboptimal rate, in particular in terms of $\log n$ terms for $\beta<2$, and since we need to control accurately the number of arms at a given distance from $\bar \mu^*$ at the same time as their empirical means.
%The difficulty of the proof is to define correctly the event of high probability, in order to avoir  
\end{proof}

% \begin{remark}
\paragraph{Discussion}
% \normalfont
The bound we obtain is minimax optimal for $\beta<2$ \textit{without additional $\log n$ factors}. We emphasize it since the previous results on infinitely many armed bandits give results which are optimal up to a $\polylog n$ factor for the cumulative regret,  except the one by~\citet{bonald2013two-target} which considers a very specific and fully parametric setting. 
For $\beta\geq 2$, our result is optimal up to a $\polylog n$ factor. We conjecture that the lower bound of Theorem~\ref{thm:lb} for $\beta \geq 2$ can be improved to $(\log(n)/n)^{1/\beta}$ and that SiRI is actually optimal up to a $\polyloglog(n)$ factor for $\beta > 2$.\\
 %And for $\beta = 2$, it is optimal also up to a $\polylog n$ factor- in some sense, being optimal in the $\beta = 2$ case is difficult since it is a point of change of regime, see Remark~\ref{rem:refim}.
% \end{remark}
\vspace{-1.5em}
\section{Extensions of SiRI}
We now discuss briefly three extensions of the SiRI algorithm that are very relevant either for practical or computational reasons, or for a comparison with the prior results. In particular,  we consider the cases 1) when $\beta$ is unknown, 2) in a natural setting where the near-optimal arms have a small variance, and 3) in the case of unknown time horizon.
These extensions are all in some sense following from our results and from the existing literature, and we will therefore state them as corollaries.

\subsection{Case of distributions on $[0,1]$ with $\bar \mu^* = 1$}
The first extension concerns the specific setting, particularly highlighted by~\citet{bonald2013two-target} but also presented by~\citet{berry1997bandits} and \citet{wang2008algorithms}, where the domain of the distributions of the arms are included in $[0,1]$ and where $\bar \mu^* = 1$. In this case, the information theoretic complexity of the problem is smaller than the one of the general problem stated in Theorem~\ref{thm:lb}. Specifically, the variance of the near-optimal arms is very small, i.e.,~in the order of $\varepsilon$ for an $\varepsilon$-optimal arm. This implies a better bound, in particular, that the parametric limitation of $1/\sqrt{n}$ can be circumvented. In order to prove it, the simplest way is to modify SiRI into Bernstein-SiRI, displayed in~Algorithm~\ref{alg:sirimber}. It is an \textit{Empirical Bernstein-modified SiRI} algorithm that accommodates the situation of distributions of support included in $[0,1]$ with $\bar \mu^* = 1$. Note that in the general case, it would provide similar results as what is provided in Theorem~\ref{thm:ub}.

\begin{algorithm}[t]
\begin{algorithmic}
\STATE {\bf Parameters:} $C, \beta, \delta$

\STATE \textbf{Newly defined quantities:} 

\STATE Set the number of arms as
\begin{align*}%\label{eq:Tb2}
\bar T_{\beta} = \lceil \min(n/\log(n), A(n) n^{\beta/2})\rceil,
\end{align*}
\STATE Modify the SiRI algorithm's UCB~\eqref{eq:UCBSiRI} with
\begin{align*}%\label{eq:UCB2}
B_{k,t} &\gets \widehat \mu_{k,t}+  2\widehat \sigma_{k,t}\sqrt{\frac{C}{T_{k,t}}\log\big(2^{2\bar t_{\beta}/b}/(T_{k,t}\delta)\big)}\\ &+ 
\frac{4C}{T_{k,t}}\log\left(2^{2\bar t_{\beta}/b}/(T_{k,t}\delta)\right),
\end{align*}
where $\widehat \sigma_{k,t}^2$ is the empirical variance, defined as
\[
\widehat \sigma_{k,t}^2 = \frac{1}{T_{k,t}} \sum_{l = 1}^{T_{k,t}} \left(X_{k,t} - \widehat \mu_{k,t}\right)^2.
\]
\STATE \textbf{Call SiRI:}
\STATE Run SiRI on the samples using these new parameters
\end{algorithmic}
\caption{Bernstein-SiRI}
\label{alg:sirimber}
\end{algorithm}

A similar idea was already introduced by~\citet{wang2008algorithms} in the infinitely many armed setting for \textit{cumulative regret}. The idea is that the confidence term is more refined using the empirical variance and hence it will be very large for a 
near-optimal arm, thereby enhancing exploration. Plugging this term in the proof, conditioning on the event of high probability, such that $\widehat \sigma_{k,t}^2$ is close to the true variance, and using similar ideas as~\citet{wang2008algorithms}, we can immediately deduce the following corollary.
\begin{corollary}\label{cor:bern}
Let $\delta>0$. Assume  Assumptions~\ref{ass:sample} and~\ref{ass:reg} of the model and that $n$ is larger than a large constant that depends on $\beta, \tilde E, \tilde E', \tilde B, C$. Furthermore, assume that all the arms have distributions of support included in $[0,1]$ and that $\bar \mu^* = 1$. Depending on~$\beta$, we have the following results for Bernstein-SiRI. %, where $E$ is a large enough constant. 
\begin{itemize}
\item Case $\beta\leq 1${\rm :} The order of the simple regret is with high probability
\begin{align*}
r_n = \cO\left(\tfrac{1}{n} \polylog n\right).
\end{align*}
\item Case $\beta> 1${\rm :} The order of the simple regret is with high probability
\begin{align*}
r_n = \cO\left(\left(\tfrac{1}{n}\right)^{1/\beta} \polylog n \right).
\end{align*}
Moreover, the rate
$$\max\left(\frac{1}{n},\left(\tfrac{\log n}{n}\right)^{1/\beta}\right),$$
is minimax-optimal for this problem, i.e.,~there exists no algorithm that achieves a better simple regret in a minimax sense.
\end{itemize}

%The the simple regret of the modified SiRI algorithm can be bounded with probability larger than $1-\delta$ as
\end{corollary}
The proof follows immediately from the proof of Theorem~\ref{thm:ub} using the empirical Bernstein bound as by~\citet{wang2008algorithms}. Moreover, the lower bounds' rates follow directly from the two facts: 1) $1/n$ is clearly a lower bound, and therefore optimal for $\beta <1$, since it takes at least $n$ samples of a Bernoulli arm that is constant times $1/n$ suboptimal, in order to discover that it is not optimal, and 2) $n^{-1/\beta}$ can be trivially deduced from Theorem~\ref{thm:lb}\footnote{Indeed, its proof shows that a lower bound of the order of $n^{-1/\beta}$ is valid for any distribution and in particular for Bernoulli with mean $\mu$ and $\bar \mu^* = 1$, which is a special case of distributions of support included in $[0,1]$ and that $\bar \mu^*=1$.}. Bernstein-SiRI is thus minimax optimal for $\beta \geq 1$ up to a $\polylog n$ factor.

\paragraph{Discussion}
% \begin{remark}
% \normalfont
%Unlike what is described in Remark~\ref{rem:refim}, the point of change of regime for Corollary~\ref{cor:bern} is $\beta = 1$, similar to what happens in the case of cumulative regret.
Corollary~\ref{cor:bern} improves the results of  Theorem~\ref{thm:ub} when $\beta \in (0,2)$. For these $\beta$, it is possible to beat the parametric  rate of $1/\sqrt{n}$, since in this case, the variance of the arms decays with the quality of the arms.
In this situation, for $\beta<2$, it is possible to beat the parametric rate $1/\sqrt{n}$ and keep the rate of $n^{-1/\beta}$ until $\beta \leq 1$, where the limiting rate of $1/n$ imposes its limitations: the regret cannot be smaller than the second order parametric rate of $1/n$. Here, the change point of regime is~$\beta = 1$ which differs from the general simple regret case but is the same as the general case of cumulative regret as discussed in Remark~\ref{rem:refim}. Notice that this comes from the fact that the limiting rate is now $1/n$ and not for same reasons as for the cumulative regret.
% \end{remark}

% \begin{remark}
% \normalfont
%Bernstein-SiRI is minimax optimal for $\beta \geq 1$ up to a $\polylog n$ factor.
% \end{remark}

\subsection{Dealing with unknown $\beta$}

In practice, the parameter $\beta$ is almost never available. Yet its knowledge is crucial for the implementation of SiRI, as well as for all the cumulative regret strategies described in~\cite{berry1997bandits,wang2008algorithms,bonald2013two-target}. Consequently, a very important question is whether it is possible to estimate it well enough to obtain good results, which we answer in the affirmative.

An interesting remark is that Assumption~\ref{ass:reg} is actually related to assuming that the 
distribution function $\cL$ is $\beta$ regularly varying in $\bar \mu^*$. 
Therefore, $\beta$ is the tail index of the distribution function of $\cL$ and can be estimated with tools from extreme value theory~\cite{dehaan2006extreme}. Many estimators exist for estimating this tail index $\beta$, 
for instance, the popular Hill's estimate~\cite{hill1975simple}, but also Pickand's' estimate~\cite{pickland1975statistical} and others.

However, our situation is slightly different from the one where the convergence of these estimators is proved, as the means of the arms are not directly observed.  As a result, we propose another estimate, related to the estimate of~\citet{carpentier2013adaptive}, which accommodates our setting. %We assume that prior to running the algorithm SiRI, we do a preliminary learning phase for $\beta$ of length $N^2$ (where $N$ will be of order $\sqrt{n/2}$ as we will see later). The algorithm SiRI is then run on the remaining $n - N^2$ samples.
Assume that we have observed $N$ arms, and that all of these arms have been sampled~$N$ times. Let us write $\widehat m_k$ for the empirical mean estimates of the mean $m_k$ of these $N$ arms and define 
$$ \widehat m^* = \max_k \widehat m_k.$$
We further define
$$\widehat p = \frac{1}{N} \sum_{k=1}^N \mathbf 1\{\widehat m^* - \widehat m_k \leq N^{-\varepsilon}\}$$
and set
\begin{equation}\label{eq:compbeta}
\widehat \beta = - \frac{ \log \widehat p}{\varepsilon \log N}\cdot
\end{equation}
This estimate satisfies the following \textit{weak} concentration inequality 
and its proof is in 
%the full paper~\cite{carpentier2014asimple}.
Appendix~\ref{s:proofbeta}.
\begin{lemma}\label{lem:beta}
Let $\underline{\beta}$ be a lower bound on $\beta$.
If Assumptions~\ref{ass:sample} and~\ref{ass:reg} are satisfied  and 
%if $\varepsilon$ is such that  
$\varepsilon < \min(\underline{\beta}, 1/2, 1/(\underline{\beta}))$, then  with probability larger than $1-\delta$, for~$N$ larger than a constant that depends only on $\tilde B$ of Assumption~\ref{ass:reg},
\begin{align*}
|\widehat \beta - \beta | &\leq \frac{\frac{\delta^{-1/\underline{\beta}}}{\underline{\beta}} + \sqrt{\log(\frac{1}{\delta})} + \max(1,\log(\tilde E'),|\log(\tilde E)|)}{\varepsilon \log N}\\ 
&\leq \frac{c'\max(\sqrt{\log(1/\delta)}, \delta^{-1/\underline{\beta}})}{\varepsilon \log N},
\end{align*}
where $c'>0$ is a constant that depends only on $\varepsilon$ and the parameter $C$ of Assumption~\ref{ass:sample}.
\end{lemma}

\begin{algorithm}[ht]
\begin{algorithmic}
\STATE {\bf Parameters:} $C, \delta, \underline{\beta}$
\STATE \textbf{Initial phase for estimating $\beta$:} 
\STATE Let  $N \gets n^{1/4}$  and $\varepsilon \gets 1/\log\log\log(n)$.
\STATE Sample $N$ arms from the arm reservoir $N$ times
\STATE Compute $\widehat \beta$ following~\eqref{eq:compbeta}
\STATE Set
\begin{equation}\label{eq:be}
\bar \beta \gets \widehat \beta + \frac{c'\max\left(\sqrt{\log(1/\delta)}, \delta^{-1/\underline{\beta}}\right) \logloglog n}{\log n}
\end{equation}

\STATE \textbf{Call SiRI:}

\STATE Run SiRI using $\bar \beta$ instead of $\beta$ with $n - N^2 = n - \sqrt{n}$ remaining samples.

\end{algorithmic}
\caption{$\bar \beta$-SiRI: \textit{$\bar\beta$-modified SiRI for unknown $\beta$}}
\label{alg:sirim}
\end{algorithm}

Let us now modify SiRI in the way as in Algorithm~\ref{alg:sirim}. The knowledge of $\beta$ is not anymore required, and one just needs a lower bound $\underline \beta$ on $\beta$. 
%Let us set $N := \sqrt{n/2}$ and us the $N^2 = n/4$. 
We get $\bar\beta$-SiRI which satisfies the following corollary.
\begin{corollary}\label{cor:beta}
Let the Assumptions~\ref{ass:sample} and~\ref{ass:reg} be satisfied. If~$n$ is large enough with respect to a constant that depends on $\beta, \tilde E, \tilde E', \tilde B, C$, then $\bar\beta-$SiRI satisfies the following:
\begin{itemize}
\item Case $\beta< 2${\rm :} The order of the simple regret is with high probability
\begin{align*}
r_n = \cO\left(\tfrac{1}{\sqrt{n}} \polyloglog n\right).
\end{align*}

\item Case $\beta> 2${\rm :} The order of the simple regret is with high probability
\begin{align*}
r_n = \cO\left(\left(\tfrac{\log n }{n}\right)^{1/\beta} \polyloglog n \right).
\end{align*}

\item Case $\beta =  2${\rm :} The order of the simple regret is with high probability
\begin{align*}
r_n = \cO\left(\tfrac{\log n}{\sqrt{n}} \polyloglog n \right).
\end{align*}
\end{itemize}
\end{corollary}
% \vspace{-0.5em}
The proof can be deduced easily from Theorem~\ref{thm:ub} using the result from Lemma~\ref{lem:beta}, noting that a $1/\log n$ rate in learning~$\beta$ is fast enough to guarantee that all bounds will only be modified by a constant factor when we use $\widehat \beta$ instead of $\beta$ in the exponent.
% \vspace{-0.5em}
\paragraph{Discussion}
% \begin{remark}
% \normalfont
Corollary~\ref{cor:beta} implies that even in situations with unknown $\beta$, it is 
possible to estimate it accurately enough so that the modified $\bar\beta$-SiRI remains minimax-optimal up to a $\polylog n$, by only using  a lower bound $\underline \beta$ on $\beta$. This is the same that holds for SiRI with known $\beta$. We would like to emphasize that $\bar \beta$ estimate~\eqref{eq:be} of $\beta$ can be used to improve cumulative regret algorithms that need $\beta$, such as the ones by~\citet{berry1997bandits} and \citet{wang2008algorithms}.
 Similarly for these algorithms, one should spend a preliminary phase of $N^2 = \sqrt{n}$ rounds to estimate $\beta$ and then run the algorithm of choice. 
This will modify the cumulative regret rates in the 
general setting by only a $\polyloglog n$ factor, which suggests that our $\beta$ estimation can be useful beyond the scope of this paper. For instance, consider the cumulative regret rate of  UCB-F by~\citet{wang2008algorithms}. If UCB-F uses our estimate of $\beta$ instead of the true $\beta$, it would still satisfy 
$$\EE{R_n}  = \cO\left(\max\left(n^{\frac{\beta}{\beta+1}} \polylog n, \sqrt{n}\polylog n\right) \right).$$

Finally, this modification can be used to prove that this problem is learnable over all mean reservoir distributions with $\beta>0$: This can be seen by setting the lower bound on $\beta$ as $\underline{\beta} = 1/\log\log\log N$, which goes to $0$ but very slowly with $n$. In this case, we only loose a $\log\log(n)$ factor.

% \end{remark}
%, and a relaxation of the classical Hall's condition in extreme value theory Estimating t
% \vspace{-0.5em}
\subsection{Anytime algorithm}

Another interesting question is whether it is possible to make SiRI anytime. This question can be quickly answered in the affirmative. 
First, we can easily just use a doubling trick to double the size of the sample in each period and throw away the preliminary samples that were used in the previous period. 
Second, \citet{wang2008algorithms} propose a more refined way to deal with an unknown time horizon (UCB-AIR), 
that also directly applies to SiRI. Using these modifications it is straightforward to transform SiRI into an anytime algorithm. The simple regret in this anytime setting will only be worsened by a $\polylog n$, where $n$ is the unknown horizon. 
Specifically, in the anytime setting, the regret of SiRI modified either using the doubling trick or by the construction of UCB-AIR has a simple regret that satisfies with high probability
\[r_n = \cO\left(\polylog (n) \max(n^{-1/2},  n^{-1/\beta} \polylog n)\right).\]
 %The only conceptual complication comes if one wants to consider an unknown $\beta$. In this case, the simplest is to do a doubling trick based algorithm that we call Anytime-SIRI (see Figure~\ref{alg:sirimany}).

% \vspace{-0.5em}
\section{Numerical simulations}
\label{sec:exp}

To simulate different regimes of the performance according to 
$\beta$-regularity, we consider different reservoir distributions 
of the arms. In particular, we consider beta distributions ${\rm B}(x,y)$ with as $x=1$ and $y = \beta$. For ${\rm B}(1,\beta)$, the Assumption~\ref{ass:reg} is satisfied precisely with regularity $\beta$.
Since to our best knowledge, SiRI is the first algorithm
optimizing \textit{simple} regret in the infinitely many arms setting, 
there is no natural competitor for it. Nonetheless, in our experiments
we compare to the algorithms designed for linked settings.

First such comparator is  UCB-F~\cite{wang2008algorithms}, an algorithm
that optimizes \textit{cumulative} regret for this setting.
UCB-F is designed for fixed horizon of $n$ evaluations 
and it is an extension of a version of~\mbox{UCB-V} by~\citet{audibert2007tuning}.
%Second, we compare SiRI to UCB-E~\cite{audibert2010best}, an algorithm
%also designed for the simple regret, but for a \textit{finite set of arms}. 
%For a comparable evaluation, we set the number of arms exactly to the
%same value as for SiRI. Note that UCB-E recommendation strategy selects
%the arm with the \textit{highest empirical regret} as opposed to the arm
%that was pulled most often. The purpose of comparison 
%of with UCB-E it to show that SiRI performs at par with UCB-E equipped 
%with the oracle. In particular, we give UCB-E the optimal number of $\bar T_\beta$
%arms and given the means of those arms, we give the exact information
%of the complexity constant~$H_1$~\cite{audibert2010best}. 
Second, we compare SiRI to lil'UCB~\cite{jamieson2014lilUCB} designed 
for the best-arm identification in the fixed confidence setting. 
The purpose of comparison with lil'UCB is to show that SiRI performs at par with 
lil'UCB equipped with  the optimal number of $\bar T_\beta$ arms. 
In all our experiments, we set constant $A$ of SiRI to $0.3$,  
constant $C$ to 1, and confidence~$\delta$ to 0.01.

All the experiments have some specific beta distribution as a reservoir and 
the arm pulls are noised with $\cN(0,1)$ truncated to $[0,1]$.  
We perform 3 experiments based on different regimes of $\beta$ coming from our analysis: 
$\beta<2$, $\beta=2$, and $\beta>2$. In the first experiment (Figure~\ref{fig:beta1}, left) we take $\beta=1$, i.e.,~${\rm B}(1,1)$ which is just a uniform distribution. In the second experiment (Figure~\ref{fig:beta1}, right) we consider ${\rm B}(1,2)$ as the reservoir. Finally, Figure~\ref{fig:beta3} features
the experiments for ${\rm B}(1,3)$.  The first obvious observation confirming the analysis is that higher $\beta$ leads to a more difficult problem. Second, UCB-F 
performs well for $\beta=1$, slightly worse for $\beta=2$, 
and much worse for $\beta=3$. This empirically
confirms our discussion in Remark~\ref{rem:refim}.
Finally, SiRI performs empirically as well as lil'UCB equipped 
with the optimal number of arms and the same confidence~$\delta$.
\begin{figure}
\begin{center}
%\vspace{-0.75em}
%\includegraphics[width=0.49\columnwidth]{fig/new6compbeta1}
%\includegraphics[width=0.49\columnwidth]{fig/new6compbeta2}
\includegraphics[width=0.49\columnwidth]{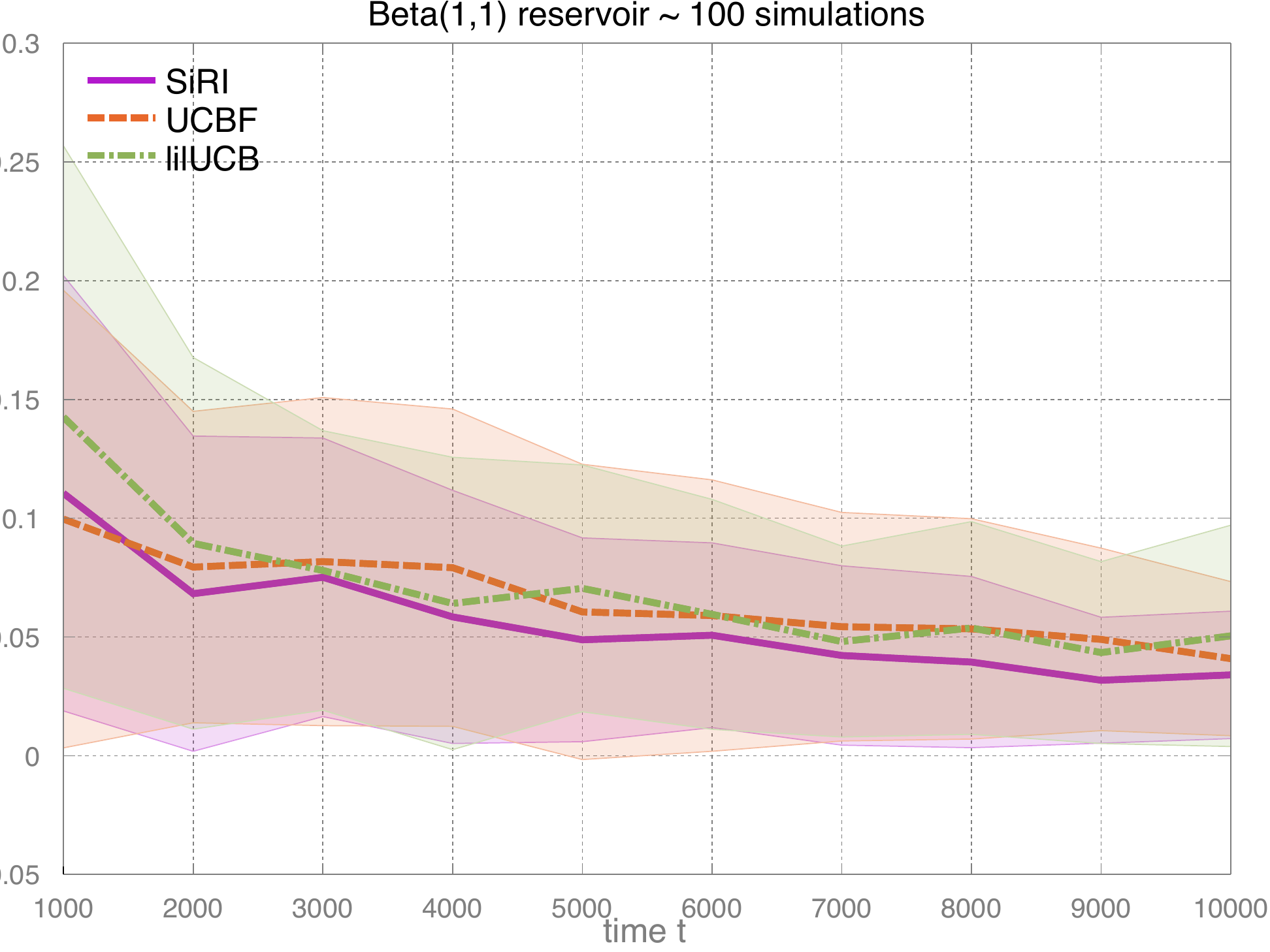}
\includegraphics[width=0.49\columnwidth]{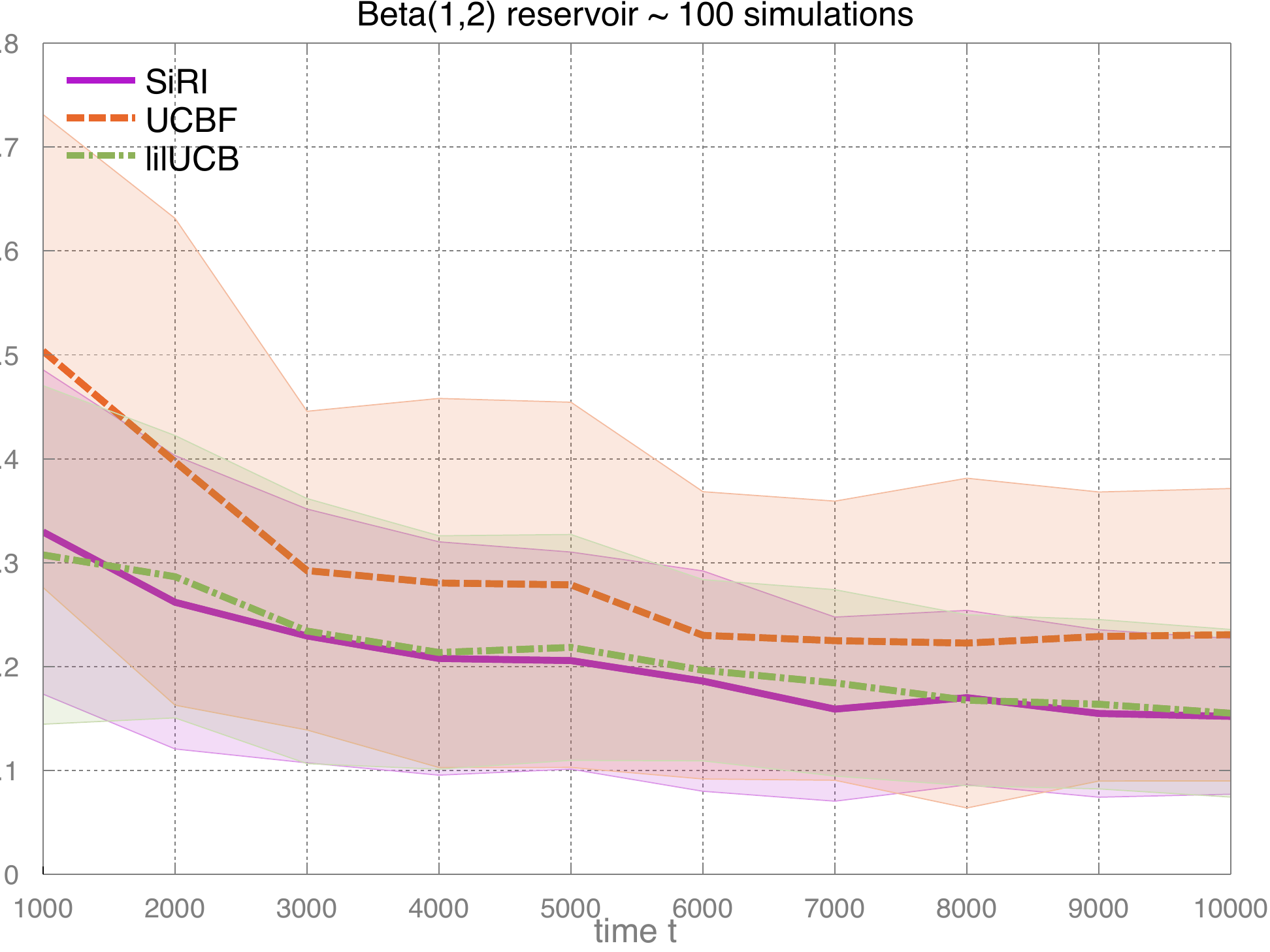}
 \vspace{-1.0em}
\caption{Uniform and ${\rm B}(1,2)$ reservoir distribution}
 \vspace{-1.em}
\label{fig:beta1}
\end{center}
\end{figure}
\begin{figure}
\begin{center}
%\hspace{0.1em}
 %\vspace{-0.2em}
\includegraphics[width=0.49\columnwidth]{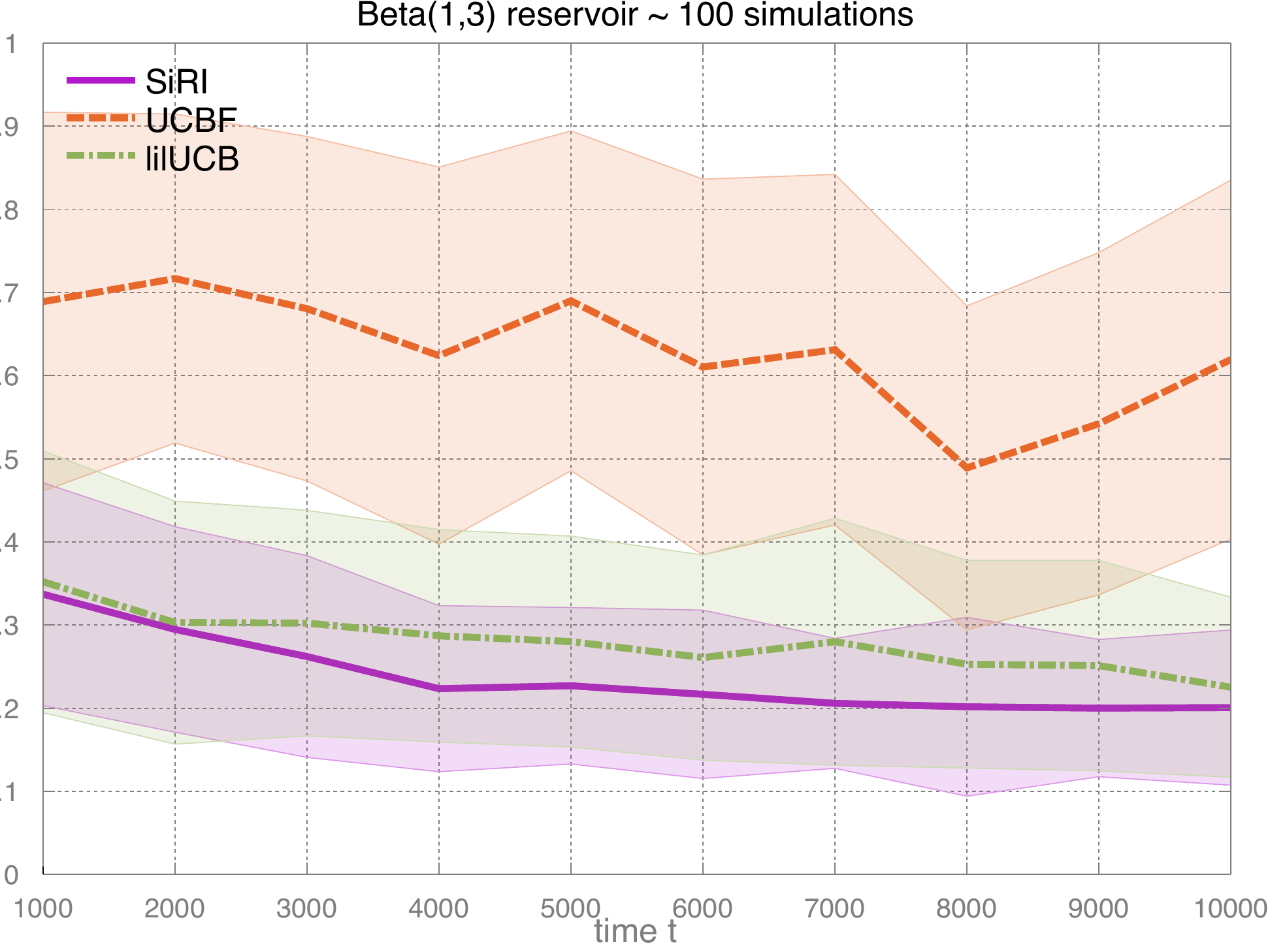}
 %\hspace{-1.2em}
\includegraphics[width=0.49\columnwidth]{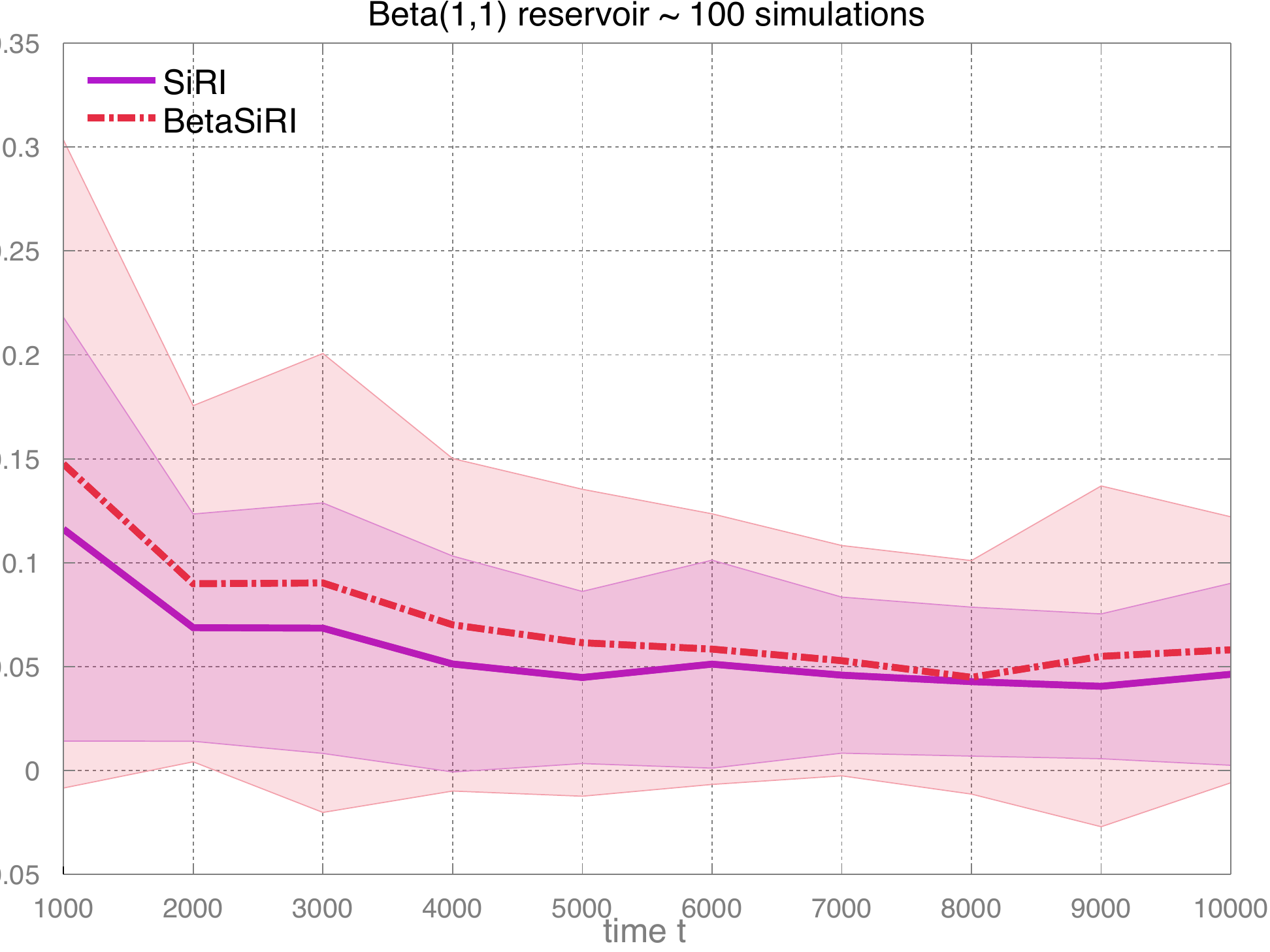}
 \vspace{-1.0em}
\caption{Comparison on ${\rm B}(1,3)$ and unknown $\beta$ on ${\rm B}(1,1)$}
\vspace{-1.5em}
\label{fig:beta3}
\end{center}
\end{figure}
% \begin{figure}
% \begin{center}
% \caption{Comparison ${\rm B}(1,2)$}
% \vspace{-1em}
% \label{fig:beta2}
% \end{center}
% \end{figure}
% \todom{doesn't it look suspicious that the simple regret for UCBF is not decreasing?}
Figure~\ref{fig:beta3} also compares SiRI with $\bar \beta$-SiRI for the uniform distribution. For this experiment, using $\sqrt{n}$ samples just for the $\beta$ estimation did not decrease the budget too much and at the same time, the estimated
$\bar\beta$ was precise enough not to hurt the final simple regret.
% \begin{figure}
% \begin{center}
% \includegraphics[width=\columnwidth]{fig/new1compbeta2}
% \caption{Comparison ${\rm B}(1,2)$}
% \vspace{-1em}
% \label{fig:beta1}
% \end{center}
% \end{figure}
 \vspace{-1.0em}
\paragraph{Conclusion}  We presented SiRI, a minimax optimal
algorithm for simple regret in infinitely many arms bandit setting, 
which is interesting when we face enormous number of potential actions.
Both the lower and upper bounds give different regimes depending on a 
complexity~$\beta$, a parameter for which we also give an efficient 
estimation procedure.
%\paragraph{Acknowledgements} 
 \vspace{-2.0em}
\paragraph*{Acknowledgments}
This work was supported by the French Ministry of Higher Education and Research and the French National Research Agency (ANR) under project ExTra-Learn n.ANR-14-CE24-0010-01.
% \begin{figure}
% \begin{center}
% \includegraphics[width=0.32\columnwidth]{fig/new5compbeta1}
% \includegraphics[width=0.32\columnwidth]{fig/new5compbeta2}
% \includegraphics[width=0.32\columnwidth]{fig/new5compbeta3}
% \caption{${\rm B}(1,1)$, ${\rm B}(1,2)$, and ${\rm B}(1,1)$}
%  \vspace{-1em}
% \label{fig:beta3}
% \end{center}
% \end{figure}
\bibliography{library}
\bibliographystyle{icml2015}
%\bibliography{library}
% \bibliography{library}
%\balance
%\bibliography{../../library}

% \todom[inline]{
% TODO: tune C and A, and try to reduce the confidence intervals artificially for SiRI
% TODO: talk about Gaussian measurements and that this does not fit the UCBF 
% }
% 
% \begin{figure}
% \begin{center}
% \includegraphics[width=0.80\columnwidth]{fig/new4compbeta1}
% \includegraphics[width=0.80\columnwidth]{fig/new4compbeta2}
% \includegraphics[width=0.80\columnwidth]{fig/new4compbeta3}
% \caption{${\rm B}(1,1)$, ${\rm B}(1,2)$, and ${\rm B}(1,1)$}
% \vspace{-1em}
% \label{fig:beta3S}
% \end{center}
% \end{figure}

% \end{document} 

\appendix

\onecolumn

\section{Additional notation}

We write $\mathbb P_1$ for the probability with respect to the arm reservoir distribution, $\mathbb P_{2}$ for the probability with respect to the distribution of the samples from the arms, and $\mathbb P_{1,2}$ for the probability both with respect to the arm reservoir distribution and the distribution of the samples from the arms.

Let $F$ be the distribution function of the \textit{mean reservoir distribution} $\mathcal L$. Let $F^{-1}$ be the \textit{pseudo-inverse} of the mean reservoir distribution. In order to express the regularity assumption, we define
\[G\left(\cdot\right) = \bar \mu^* - F^{-1}\left(1-\cdot\right).\]
%and
%$$\bar F(.) = 1 - F(\bar \mu^* - .).$$\todom[inline]{we never used $\bar F(.)$}
We assume that $G$ has a certain regularity in its right end point, which is a standard assumption for infinitely many armed bandits.
In particular, we rewrite Assumption~\ref{ass:reg} by only modifying the constants $\tilde E,\tilde E',$ and $\tilde B$.
\begin{assumption}[$\beta$ regularity in $\bar \mu^*$, version 2]\label{ass:reg2}
Let $\beta>0$. There exist $E,E',B \in(0,1) $ such that  $\forall u \in[0,B]$,
$$E' u^{1/\beta} \geq G\left(u\right)\geq E u^{1/\beta}.$$
\end{assumption}
This assumption is equivalent to Assumption~\ref{ass:reg} which is the same as the classic one \eqref{eq:refreg} by definition of $G$ and $F$ and we reformulate it  for the convenience of analysis. Without loss of generality, we assume that $\bar \mu^*>0$.

\section{Full proof of Theorem~\ref{thm:ub}}\label{ub}

\subsection{Roadmap}

The proof of Theorem~\ref{thm:ub} (upper bounds) is composed of two layers. The first layer consists of proving results on the empirical distributions of the arms emitted by the arms reservoir, the crucial object is  event $\xi_1$.  The second layer consists of proving results on the random samples of the arms, and in particular that the empirical means of the arms are not too different from the true means of the arms. For this part, the crucial object is event $\xi_2$. More precisely, these two layers can be decomposed as follows. 
\begin{itemize}
%\item A proof that among the $\bar T_{\beta}$ arms pulled by the algorithm, there is with high probability at least one arm that has a gap with respect to $\bar \mu^*$ that is smaller than the simple regret (from Theorem~\ref{thm:ub}) - this is done in Corollary~\ref{cor:barm}.
\item We prove of suitable high probability upper bounds and lower bounds on the number of arms among the $\bar T_{\beta}$ arms pulled by the algorithm that have a given gap (with respect to $\bar \mu^*$), depending on the considered gap.  This is done in Lemma~\ref{lem:xi1}. Two important results can be consequently deduced: (i) An upper bound on the number of suboptimal arms depending on how suboptimal they are. The more suboptimal they are, the more arms they are, which depends on $\beta$. (ii) A proof that among the $\bar T_{\beta}$ arms pulled by the algorithm, there is with high probability at least one arm, and not significantly more than one arm, that has a gap smaller than the simple regret from Theorem~\ref{thm:ub}. This is done in Corollary~\ref{cor:barm}.
\item In Lemma~\ref{lem:xi2}, we prove that with high probability, the empirical means of the arms are not too different from their true means. The main difficulty is that the means of the arms are random. In order  to avoid suboptimal $\log(n)$ dependency in the case $\beta<2$,  we use a peeling argument where the peeling is done over these random gaps, using the result from the previous layer, i.e.,~the bound on the number of arms with a given gap.
\end{itemize}
Afterwards, we combine the two results to bound the number of suboptimal pulls (section~\ref{ss:cool}). Since the algorithm pulls the arms depending on the empirical gaps, then (i) the bounds on the number of suboptimal and near-optimal arms, and (ii) the bounds on the deviations of the empirical means with respect to the true means, will allow to obtain the desired bound on the number of suboptimal arms. By construction of the strategy and in particular, by the choice of $\bar T_{\beta}$, we prove that with high probability, the number of pulls of the optimal arms is smaller than a fraction of $n$. This means that there is a near-optimal arm that is pulled more than $n/2$ times. This is the one selected by the algorithm which concludes the proof.

\subsection{Concentration inequalities}
We make several uses of \textit{Bernstein's inequality}:
\begin{lemma}[Bernstein's inequality]\label{lemma:bernstein}
Let $\E(X_t) = 0, |X_t|\leq b > 0$, and 
$\E(X_t^2)\leq v > 0$. 
% Then for any $\varepsilon > 0$
% \[
% P\left( \sum_{t=1}^n X_t \ge \varepsilon) \exp\left( -\frac{\varepsilon^2}{2nv + 2b\varepsilon/3} \right)\right) 
% \]
Then for any $\delta > 0$, with probability at least $1-\delta$
\[
 \sum_{t=1}^n X_t \leq \sqrt{2nv \log\delta^{-1}} + \tfrac{b}{3}\log\delta^{-1}
\]
\end{lemma}

Furthermore, Algorithm~\ref{alg:sirimber} along with 
Corollary~\ref{cor:bern} are based on the \textit{empirical Bernstein concentration inequality}.

\begin{lemma}[Empirical Bernstein's inequality]\label{lemma:ebernstein}
 Let $\E(X_t) = 0, |X_t|\leq b > 0$. Let for any $j = 1, \dots, n$
\[
 V_j = \frac{1}{j}\sum_{t=1}^j\left( X_t - \frac{1}{j}\sum_{i=1}^j X_i \right)^2
\]

Then for any $\delta > 0$, with probability at least $1-\delta$
\[
 \sum_{t=1}^j X_t \leq \sqrt{2n V_j \log\left(3\delta^{-1}\right)} + 3b\log\left(3\delta^{-1}\right)
\]
\end{lemma}

\subsection{Notation}

For any $i \leq K_n$, set
$$\Delta_i = \bar \mu^* - \mu_i,$$
where we remind the reader that $\mu_i$ is the mean of distribution of arm $i$.

%In our analysis, we make several uses of a peeling argument.
%Specifically, we divide the domain of the arms' mean into segment depending on their distance from the optimal mean.  
%The smaller size intervals will be given more weight when we apply a weighted union bound.

Without loss of generality, we assume that $\bar \mu^*>0$. For any $u \in \mathbb N$, we define
$$I_u = \left[\bar \mu^* -G\left(2^{-u}\right), \bar \mu^* -G\left(2^{-u-1}\right)\right].$$
We also define
$$I_{-1} = \left[0, \bar \mu^* - G\left(B\right)\right].$$
We further define
$$I^* = \left[\bar \mu^* -G\left(2^{-\bar t_{\beta}}\right), \bar \mu^* -G\left(0\right)\right] = \left[\bar \mu^* -G\left(2^{-\bar t_{\beta}}\right), \bar \mu^*\right].$$
Let $N_u$ be the number of arms in segment $I_u$, 
$$N_u = \sum_{k=1}^{\bar T_{\beta}} \textbf{1}\{\mu_k \in I_u\},$$
and let $\bar N^*$ be the number of arms in the segment $I^*$, 
$$\bar N^* = \sum_{k=1}^{\bar T_{\beta}} \textbf{1}\{\mu_k \in I^*\}.$$

\subsection{Favorable high-probability event}

Let $\xi_1$ be the event defined as
\begin{align*}
\xi_1 &=\Bigg\{\omega: \forall u\in \mathbb N, u \leq \bar t_{\beta}, \left|N_u - 2^{\bar t_{\beta}-u-1}\right| \leq \sqrt{(\bar t_{\beta}-u+1) 2^{\bar t_{\beta}-u}\log(1/\delta)} + (\bar t_{\beta}-u+1) \log(1/\delta), \\
&\qquad\qquad\qquad \mathrm{\ and\ }  \bar N^* \leq 1 + 2\sqrt{\log(1/\delta)} + 2 \log(1/\delta)\Bigg\}\\
&=\Bigg\{\omega: \forall u\in \mathbb N, u \leq \bar t_{\beta}, \left|N_u - 2^{\bar t_{\beta}-u-1}\right| \leq 2^{\bar t_{\beta}-u-1}\varepsilon_u  \mathrm{\ and\ } \bar N^* \leq 1 + \varepsilon_{\bar t_{\beta}}\Bigg\}.
\end{align*}
where $ \varepsilon_u = 2\sqrt{(\bar t_{\beta}-u+1) 2^{-(\bar t_{\beta}-u)}\log(1/\delta)} + 2(\bar t_{\beta}-u+1) 2^{-(\bar t_{\beta}-u)} \log(1/\delta)$.

\begin{lemma}\label{lem:xi1}
The probability of $\xi_1$ under both the distribution of the arm reservoir and the samples of the arms is larger than $1-\left(1+\frac{e}{e-1}\right)\delta$ for $\delta$ small enough, 
$$\mathbb P_1(\xi_1) = \mathbb P_{1,2}\left(\xi_1\right) \geq 1-\left(1+\frac{e}{e-1}\right)\delta.$$ 
\end{lemma}
\begin{proof}[Proof of Lemma~\ref{lem:xi1}]
Let $u\in \mathbb N$. We have by definition that
$$N_u = \sum_{k=1}^{\bar T_{\beta}} \textbf{1}\{\mu_k \in I_u\},$$
is a sum of independent Bernoulli random variables of parameter $2^{-u} - 2^{-u-1} = 2^{-u-1}$. 
By a Bernstein concentration inequality~(Lemma~\ref{lemma:bernstein}) for sums of Bernoulli random variables, this implies that with probability $1-\delta_u>0$,
\begin{align*}
\left|N_u - 2^{\bar t_{\beta}-u-1}\right| \leq \sqrt{ 2^{\bar t_{\beta}-u}\log\delta_u^{-1}} + \log\delta_u^{-1}.
\end{align*}
Set $\delta_u = \exp\left(-\left(\bar t_{\beta} - u+1\right)\right)\delta$. Notice that for $u\leq \bar t_{\beta}$, $\log \delta_u^{-1} \leq (\bar t_{\beta}-u+1) \log \delta^{-1}$. Then the result holds by a union bound since
for $\delta$ small enough
\[\sum_{u=0}^{\bar t_{\beta}} \delta_u = \delta \sum_{u=0}^{\bar t_{\beta}} \exp\left(-\left(\bar t_{\beta} - u+1\right)\right) \leq \frac{e\delta}{e-1},\]
and by similar argument for $\bar N^* $ which 
 together with another union bound give the claim.
\end{proof}

The following corollary follows from Lemma~\ref{lem:xi1}.
\begin{corollary}\label{cor:barm}
Set $\bar t^* = \lfloor \bar t_{\beta} - 96 \log_2(\log_2(\log(1/\delta))) - \log_2(\log(1/\delta))\rfloor -2$. 
Let $\delta$ be  smaller than an universal constant. 
If $n$ is large enough so that $\bar t^*\geq \log_2(1/B)$, then on $\xi_1$, there is at least an arm of index in $\{1, \ldots, \bar T_{\beta}\}$ such that it belongs to $I_{\bar t^*}$. If $k^*$ is its index, then
$$\Delta_{k^*} \leq  \tfrac{1}{4} E'(\log_2(\log(1/\delta)))^{96} 2^{-\bar t_{\beta}/\beta} \log(1/\delta).$$
\end{corollary}
\begin{proof}[Proof of Corollary~\ref{cor:barm}]
First we have for $u \leq \bar t^*$ %\todom{correct direction?}
%\todom[inline]{the following derivation would be good to recheck}
\begin{align*}
\varepsilon_u &= 2\sqrt{(\bar t_{\beta}-u+1) 2^{-(\bar t_{\beta}-u)} \log(1/\delta)} + 2 (\bar t_{\beta}-u+1) 2^{-(\bar t_{\beta}-u)} \log(1/\delta)\\
&\leq 2 \sqrt{\frac{(1 + \log_2(\log(1/\delta)) +  96 \log_2(\log_2(\log(1/\delta))) ) }{96\log_2(\log(1/\delta))} } 
+ \frac{2(1 + \log_2(\log(1/\delta)) +  96 \log_2(\log_2(\log(1/\delta))) ) }{96\log_2(\log(1/\delta))} \\
&\leq  4\sqrt{1/(96\log_2(\log(1/\delta))) +1/96 + \log_2(\log_2(\log(1/\delta)))/\log_2(\log(1/\delta))}\\
&\leq 1/2, \end{align*}
%\todom{still only $<1$}
for $\delta$ being a small enough constant. %\todom{$4 + \log_2$ instead of $1 + \log_2$? }

This implies that for $u \leq \bar t^*$
\begin{align*}
2^{\bar t_{\beta}-u-1}(1-\varepsilon_u) \geq 2^{2 - 1} \times 1/2 \geq 1.
\end{align*}%\todom{still holds?}
This implies that on $\xi_1$, $N_{\bar t^*} \geq 1$, which means there is at least one arm in $I_{\bar t^*}$. Let us call $k$ one of these arms. 
By definition of $I_{\bar t^*}$,  it satisfies 
$$\Delta_{k^*} \leq G(2^{-\bar t^*}) \leq E'2^{-\bar t^*/\beta}\leq \tfrac{1}{4}E'(\log_2(\log(1/\delta)))^{96} 2^{-\bar t_{\beta}/\beta} \log(1/\delta).$$
because of Assumption~\ref{ass:reg2}, since $t^* \geq \log_2(1/B)$.
\end{proof}

Let for any $k\in \mathbb N, 1 \leq k \leq \bar T_{\beta}$
$$n_k = \left\lfloor \log_2\left(\frac{D\log\left(\max(1,2^{2\bar t_{\beta}/b}\Delta_k^2)/b\right)}{\max\left(2^{-2\bar t_{\beta}/b},\Delta_k^2\right)} \right) \right\rfloor,$$
where $D$ is a large constant, and
$$\tilde n_u =  \log_2\left(\frac{D\log\left(\max\left(1,2^{2\bar t_{\beta}/b}G\left(2^{-\left(u+1\right)}\right)^2\right)/\delta\right)}{\max\left(2^{-2\bar t_{\beta}/b},G\left(2^{-(u+1)}\right)^2\right)} \right).$$
Let also
$$\tilde n_{-1} = \log_2\left(\frac{D\log\left(\max(1,2^{2\bar t_{\beta}/b}G\left(B\right)^2)/\delta\right)}{\max\left(2^{-2\bar t_{\beta}/b},G\left(B\right)^2\right)} \right).$$

\begin{align*}
\mathrm{Let\ }\xi_2 =\left\{\omega: \forall k\in \mathbb N^*, k \leq \bar T_{\beta}, \forall v \leq n_k  \left|\frac{1}{2^v} \sum_{i=1}^{2^v}X_{k,i} - \mu_k \right| \leq 2\sqrt{C2^{-v}\log(2^{2\bar t_{\beta}/b-v}/\delta)} + 2C2^{-v}\log(2^{2\bar t_{\beta}/b-v}/\delta) \right\}.
\end{align*}

\begin{lemma}\label{lem:xi2}
\textbf{Case $\beta< 2${\rm:}} Knowing $\xi_1$, the probability of $\xi_2$ is larger than $1 - H\log(1/\delta)^2 \delta$, 
$$\mathbb P_2\left(\xi_2 | \xi_1\right)  \geq 1 - H\log\left(1/\delta\right)^2 \delta,$$
where $H$ is a constant that depends only on $D,E,E',\beta$.

\textbf{Case $\beta \geq 2${\rm:}} Knowing $\xi_1$, the probability of $\xi_2$ is larger than $1 - H\log(1/\delta)^2 \log(n) \delta$, 
$$\mathbb P_2\left(\xi_2 | \xi_1\right)  \geq 1 - H\log(1/\delta)^2  \log(n)^2 \delta,$$
where $H$ is a constant that depends only on $D,E,E',\beta$.

%\textbf{Case $\beta= 2$:} Knowing $\xi_1$, the probability of $\xi_2$ is larger than $1 - H\log(1/\delta)^2 \log(n)^3 \delta$, i.e.
%$$\mathbb P_2(\xi_2 | \xi_1)  \geq 1 - H\log(1/\delta)^2 \log(n)^3 \delta,$$
%where $H$ is a constant that depends only on $D,E,\beta$.
\end{lemma}
\begin{proof}[Proof of Lemma~\ref{lem:xi2}]
Let $(k, v)\in \mathbb N^*\times\mathbb N$. Since $(X_{k,i})_i$ are i.i.d.~from distribution bounded by $C$, we have that with probability (according to the samples) larger than $1-\delta_{k,v}$,
$$\left|\frac{1}{2^v} \sum_{i=1}^{2^v}X_{k,i} - \mu_k \right| \leq \sqrt{2C2^{-v}\log\left(1/\delta_{k,v}\right)} + 2C\times2^{-v}\log\left(1/\delta_{k,v}\right).$$
Set $\delta_{k,v} = 2^v 2^{-2\bar t_{\beta}/b}\delta$. We have
\begin{align*}
\sum_{k \leq \bar T_{\beta}}\sum_{v \leq n_k} \delta_{k,v} &=  2^{-2\bar t_{\beta}/b}\delta  \sum_{k \leq \bar T_{\beta}}\sum_{v \leq n_k} 2^v
\leq   2\times 2^{-2\bar t_{\beta}/b}\delta \sum_{k \leq T_{\beta}} 2^{n_k}
%&= 2\frac{\delta}{n}\sum_{k \leq \bar T_{\beta}}  \frac{D\log(\max(1,n\Delta_k^2)/\delta}{\max(1/n,\Delta_k^2)}\\
\leq  2\times 2^{-2\bar t_{\beta}/b}\delta  \sum_{u = 0}^{\infty} N_u  2^{\tilde n_u},
\end{align*}
since $2^{\tilde n_u}$ is increasing in $u$.

Again, since $2^{\tilde n_u}$ is increasing in $u$, is implies that on $\xi_1$, 
\begin{align}
\sum_{k \leq \bar T_{\beta}}\sum_{v \leq n_k} \delta_{k,v} &\leq 2\times 2^{-2\bar t_{\beta}/b}\delta \sum_{u = 0}^{\infty} N_u   2^{\tilde n_u} \nonumber\\
&\hspace{-2cm}\leq  2\times 2^{-2\bar t_{\beta}/b}\delta \left( \bar T_{\beta} 2^{\tilde n_{-1}} + \sum_{u = \lfloor \log_2(1/B)\rfloor+1}^{\bar t_{\beta}} N_u   2^{\tilde n_u} + \bar N^*   2^{\tilde n_{\bar t_{\beta}}} \right) \label{kick}\\
&\hspace{-2cm}\leq  2\times 2^{-2\bar t_{\beta}/b}\delta  \left( \frac{\bar T_{\beta}D\log(2^{2\bar t_{\beta}/b}E' B^{-2/\beta}/\delta)}{E B^{-2/\beta}} + \!\!\!\!\!\!\!\!\!\sum_{u = \lfloor \log_2(1/B)\rfloor+1}^{\bar t_{\beta}}\!\!\!\!\!\!\!\!\!  \frac{N_uD\log(2^{2\bar t_{\beta}/b}G(2^{-(u+1)})^2/\delta)}{G(2^{-(u+1)})^2}\nonumber
+ \bar N^*   D\log\left(\tfrac{1}{\delta}\right)2^{2\bar t_{\beta}/b} \right)\nonumber\\
&\hspace{-2cm}\leq 2\times 2^{-2\bar t_{\beta}/b}\delta \left( \frac{2^{\bar t_{\beta}}DE'\log\left(\tfrac{n}{\delta}\right)}{E} + \!\!\!\!\!\!\!\!\!\! \sum_{u = \lfloor \log_2(1/B)\rfloor+1}^{\bar t_{\beta}} \!\!\!\!\!\!\!\!\!\!\!\! 2^{\bar t_{\beta}\!-\! u \!-1\!}(1+\varepsilon_u)  \frac{D\log\left(\tfrac{E}{\delta}\right)\left(\frac{2\bar t_{\beta}}{b}\! -\! \frac{2(u\!-\!1)}{\beta}\right)}{E2^{-2(u-1)/\beta}}
+ (1+\varepsilon_{\bar t_{\beta}}) D \log\left(\tfrac{1}{\delta}\right)2^{2\bar t_{\beta}/b} \right)\nonumber\\
&\hspace{-2cm}\leq 2\times 2^{-2\bar t_{\beta}/b}\delta \left( \frac{2^{\bar t_{\beta}} DE'\log\left(\tfrac{n}{\delta}\right)}{E} + \sum_{u = 0}^{\bar t_{\beta}} 6D/(Eb) \log\left(\tfrac{E}{\delta}\right)^2  2^{\bar t_{\beta} - u + 2u/\beta}(2\bar t_{\beta} - 2(u -1)) 
+  5 D \log\left(\tfrac{1}{\delta}\right)^2 2^{2\bar t_{\beta}/b} \right),\nonumber
\end{align}
since $\varepsilon_u \leq 4(\bar t_{\beta}-u+1)\log\left(\tfrac{1}{\delta}\right)$ and since $b \leq \beta$, which implies that $2\bar t_{\beta} - 2(u -1)\geq 1$.

\paragraph{Case 1: $\beta < 2${\rm:}} In this case, $b = \beta$. Since $\sum_{u = 0}^{\infty} 2^{-u/v} u^{v'} <\infty$ for any $v,v'>0$ that on $\xi_1$, the last equation implies  
\begin{align*}
\sum_{k \leq \bar T_{\beta}}\sum_{v \leq n_k} \delta_{k,v}
\leq 2\times 2^{-2\bar t_{\beta}/\beta}\delta \left( \frac{2^{\bar t_{\beta}} DE'\log\left(\tfrac{n}{\delta}\right)}{E} + \frac{3DF_1'}{E\beta} \log\left(\tfrac{E}{\delta}\right)  2^{2\bar t_{\beta}/\beta}  +  5 D \log\left(\tfrac{1}{\delta}\right)^2 2^{2\bar t_{\beta}/\beta} \right)
\leq  F_1 \log\left(\tfrac{1}{\delta}\right)^2 \delta,
\end{align*}
where $F_1',F_1>0$ are constants.

\paragraph{Case 2: $\beta > 2${\rm:}} In this case, $b=2$. Since $\sum_{u = 0}^{\infty} 2^{-u/v}  <\infty$ for any $v>0$ that on $\xi_1$, the last equation implies
\begin{align*} \!\!\!\!\!\!\!\!\!\!\!\!\!\!\!\!\!\!\!\!\!\!\!\!
\!\!\!\!\!\!\!\!\!\!\!\!\!\!\!\!\!\!\!\!\!\!\!\!
\sum_{k \leq \bar T_{\beta}}\sum_{v \leq n_k} \delta_{k,v}
&\leq 2\times 2^{-\bar t_{\beta}}\delta \left(  \frac{2^{\bar t_{\beta}}DE'\log\left(\tfrac{n}{\delta}\right)}{E} +  \frac{3DF_2'}{E\beta}  \log\left(\tfrac{E}{\delta}\right)^2  2^{\bar t_{\beta}} t_{\beta}  +  5 D \log\left(\tfrac{1}{\delta}\right)^2 2^{\bar t_{\beta}} \right)\\
&\leq F_2 \log\left(\tfrac{1}{\delta}\right)^2 \bar t_{\beta} \delta  \leq F_2 \log\left(\tfrac{1}{\delta}\right)^2 \log(n)\delta,
\end{align*}
where $F_2',F_2>0$ are constants.

\paragraph{Case 3: $\beta = 2${\rm:}} In this case, we have on $\xi_1$
\begin{align*}
\sum_{k \leq \bar T_{\beta}}\sum_{v \leq n_k} \delta_{k,v} &\leq 2\times 2^{-\bar t_{\beta}}\delta \left(   \frac{2^{\bar t_{\beta}} DE'\log\left(\tfrac{n}{\delta}\right)}{E} + \sum_{u = 0}^{\bar t_{\beta}} \frac{3D}{E\beta}  \log\left(\tfrac{E}{\delta}\right)^2   2^{\bar t_{\beta}}(2\bar t_{\beta} - 2(u -1))  +  5 D \log\left(\tfrac{1}{\delta}\right)^2 2^{\bar t_{\beta}} \right)\\
&\leq F_3 \log\left(\tfrac{1}{\delta}\right)^2  \bar t_{\beta}^2 \delta \leq F_3 \log\left(\tfrac{1}{\delta}\right)^2  \log(n)^2 \delta.
\end{align*}
where $F_3>0$ is a constant.
\end{proof}

Let
$\xi = \xi_1 \cap \xi_2.$
By Lemmas~\ref{lem:xi1} and~\ref{lem:xi2}, we know that for a given constant $F_4$ that depends only on $\beta,D,E,E'$,
\begin{itemize}
\item Case $\beta<2$:
$$\mathbb P(\xi) \geq 1 - F_4\log\left(\tfrac{1}{\delta}\right)^2 \delta.$$
\item Case $\beta\geq 2$:
$$\mathbb P(\xi) \geq 1 - F_4\log\left(\tfrac{1}{\delta}\right)^2  \log(n)^3 \delta.$$
\end{itemize}

\subsection{Upper bound on the number of pulls of the non-near-optimal arms}\label{ss:cool}

Let $k$ be an arm such that $k \leq \bar T_{\beta}$, and $t \leq n$ be a time. On the event $\xi$, by definition, we have
$$\left|\widehat \mu_{k,t} - \mu_k \right| \leq 2\sqrt{\frac{C\log\big(2^{2\bar t_{\beta}/b}/(T_{k,t}\delta)\big)}{T_{k,t}}} + \frac{2C\log\big(2^{2\bar t_{\beta}/b}/(T_{k,t}\delta)\big)}{T_{k,t}},$$
which implies by definition of the upper confidence bound that on $\xi$
\begin{align}\label{bound}
\mu_k \leq B_{k,t}\leq \mu_k + \varepsilon_{k,t},
% \end{align}
\quad\mathrm{where}\quad \varepsilon_{k,t} = 2\sqrt{\frac{C\log\left(2^{2\bar t_{\beta}/b}/(T_{k,t}\delta)\right)}{T_{k,t}}} + \frac{2C\log\left(2^{2\bar t_{\beta}/b}/(T_{k,t}\delta)\right)}{T_{k,t}}.
\end{align}
Let us now write $k^*$ for the best arm among the ones in $\{1, \ldots, T_{\beta}\}$. Note that $k^*$  may be different from the best possible arm.
By Corollary~\ref{cor:barm}, we know that on $\xi$, $k^*$ satisfies 
$$\Delta_{k^*} \leq  \tfrac{1}{4} E'\left(\log_2\left(\log\left(1/\delta\right)\right)\right)^{96} 2^{-\bar t_{\beta}/\beta} \log\left(\tfrac{1}{\delta}\right) = \varepsilon^*.$$

Arm $k$ is pulled at time $t$ instead of $k^*$  only if
$$B_{k,t} \geq B_{k^*,t}.$$
On $\xi$, this happens if
\begin{align*}
\bar \mu^* - \varepsilon^* \leq \mu_k + \varepsilon_{k,t},
\end{align*}
which happens if (on $\xi$)
\begin{align*}
\Delta_k - \varepsilon^* \leq  \varepsilon_{k,t},
\end{align*}
and if we assume that $\Delta_k \geq 2\varepsilon^*$, it implies that on $\xi$ arm $k$ is pulled at time $t$ only if
\begin{align}\label{eq:goule}
\Delta_k  \leq  2\varepsilon_{k,t}.
\end{align}

We define $u$ such that (i) that $\mu_k \in I_u$ if $u \geq \lfloor \log_2(B) \rfloor+1$ or (ii) $u=-1$ otherwise. Assume that 
\begin{align*}
T_{k,t} \geq  2^{\tilde n_u} \geq  \frac{D\log\left(\max\left(1,2^{2\bar t_{\beta}/b}G\left(2^{-\left(u+1\right)}\right)^2\right)/\delta\right)}{\max\left(2^{-2\bar t_{\beta}/b},G\left(2^{-\left(u+1\right)}\right)^2\right)} 
&\geq \frac{D\log(2^{2\bar t_{\beta}/b}G(2^{-(u+1)})^2/\delta)}{G(2^{-(u+1)})^2},
\end{align*}
since we assumed that $\Delta_k \geq 2\varepsilon^*$, which implies that $t_{k,t}\leq t^*\leq \bar t_{\beta}$.

By Assumption~\ref{ass:reg2}, and since $\mu_k \in I_u$, we know that $G(2^{-(u+1)}) \leq \Delta_k$. Therefore, the last equation implies
\begin{align*}
T_{k,t} \geq  2^{\tilde n_u} \geq  \frac{D\log\left(\max\left(1,2^{2\bar t_{\beta}/b}G\left(2^{-\left(u+1\right)}\right)^2\right)/\delta\right)}{\max\left(2^{-2\bar t_{\beta}/b},G\left(2^{-\left(u+1\right)}\right)^2\right)} 
&\geq  \frac{D\log\left(2^{2\bar t_{\beta}/b}\Delta_k^2/\delta\right)}{\Delta_k^2}.
\end{align*}
For such a $T_{k,t}$, we have
\begin{align*}
\frac{\log\left(2^{2\bar t_{\beta}/b}/\left(T_{k,t}\delta\right)\right)}{T_{k,t}} &\leq \frac{\Delta_k^2\log\left(D2^{2\bar t_{\beta}/b}\Delta_k^2/\delta\right)}{D\log\left(2^{2\bar t_{\beta}/b}\Delta_k^2/\delta\right)}
\leq \frac{\Delta_k^2 \log D}{D}
\leq \Delta_k^2/(16C),
\end{align*}
for $D$ large enough so that $D \geq 32(C+1)\log(32(C+1))$. Therefore, by definition of $\varepsilon_k,t$, the last equation implies that for $T_{k,t} \geq 2^{\tilde n_u}$, we have
$$\varepsilon_{k,t} \leq \Delta_k/4.$$

The last equation implies together with~\eqref{eq:goule} that if $T_{k,t} \geq  2^{\tilde n_u}$, then on $\xi$, arm $k$ is not pulled from time $t$ onwards. In particular, this implies that on $\xi$
$$T_{k,n}\leq 2^{\tilde n_u},$$
for any $k \leq \bar T_{\beta}$ such that $\Delta_k \geq 2\varepsilon^*$, and such that (i) $\mu_k \in I_u$ if $u \geq \lfloor \log_2(B) \rfloor+1$ or (ii) or $u=-1$ otherwise.

Let $\mathcal A$ be the set of arms such that $\Delta_k \geq 2\varepsilon^*$. From the previous equation, the number of times that they are pulled is bounded on $\xi$ as
\begin{align*}
\sum_{k \in \mathcal A} T_{k,n} \leq \sum_{u \leq \bar t_{\beta}} N_u 2^{\tilde n_u}
\leq \bar T_{\beta} 2^{\tilde n_{-1}} + \!\!\!\!\!\!\!\!\sum_{\lfloor \log_2(B) \rfloor \leq u \leq \bar t_{\beta}} N_u 2^{\tilde n_u}.
\end{align*}
Bounding this quantity can be done in essentially the same way as in~\eqref{kick}. We again obtain three cases,
\begin{itemize} 
\item \textbf{Case 1: $\beta < 2${\rm:}} In this case on $\xi$
\begin{align*}
\sum_{k \in \mathcal A} T_{k,n} &\leq \frac{2^{\bar t_{\beta}} DE'\log(n/\delta)}{E} + \frac{3DF_1'}{E\beta} \log(E/\delta)^2  2^{2\bar t_{\beta}/\beta}
\leq n/H,
\end{align*}
where $H$ is arbitrarily large for $A$ small enough in the definition of $\bar T_{\beta}$.

\item \textbf{Case 2: $\beta > 2${\rm:}} In this case on $\xi$
\begin{align*}
\sum_{k \in \mathcal A} T_{k,n} &\leq  \frac{2^{\bar t_{\beta}}DE'\log(n/\delta)}{E} +  \frac{3DF_2'}{E\beta} \log(E/\delta)^2  2^{\bar t_{\beta}} t_{\beta}
\leq n/H,
\end{align*}
where $H$ is arbitrarily large for $A$ small enough in the definition of $\bar T_{\beta}$.

\item \textbf{Case 3: $\beta = 2${\rm:}} In this case on $\xi$
\begin{align*}
\sum_{k \in \mathcal A} T_{k,n} &\leq  \frac{2^{\bar t_{\beta}}  DE'\log(n/\delta)}{E} + \sum_{u = 0}^{\bar t_{\beta}} \frac{3D}{E\beta} \log(E/\delta)^2  2^{\bar t_{\beta}}(2\bar t_{\beta} - 2(u -1)) 
\leq n/H,
\end{align*}
where $H$ is arbitrarily large for $A$ small enough in the definition of $\bar T_{\beta}$.
\end{itemize}

Consider now $u^*$ such that $u^* = \lfloor \log_2\big(1/\bar F(\varepsilon^*)\big)\rfloor$. By definition of $\varepsilon^*$, we know that on $\xi$, we have
$$u^* \geq t^*-\upsilon(\delta)$$
Therefore with high probability, by Lemma~\ref{lem:xi1} and as in Corollary~\ref{cor:barm}, on $\xi$, there are less than $N(\delta)$ arms of index smaller than $\bar T_{\beta}$ such that $\Delta_k \leq 2\varepsilon^*$, where $N(\delta)$ is a constant that depends only on $\delta$. For $H$ large enough, on $\xi$, $N(\delta)\leq H$. This implies, together with the three cases, that there is at least an arm of index smaller than $\bar T_{\beta}$ and such that $\Delta_k \leq 2\varepsilon^*$ that is pulled more than $n/H$ times. This implies that the most pulled arm is such that, on $\xi$, $\Delta_k \leq 2\varepsilon^*$. This implies that the regret is  on $\xi$ bounded as
\begin{align*}
r_n \leq \frac{E'(\log_2(\log(1/\delta)))^{96}}{4} 2^{-\bar t_{\beta}/\beta} \log(1/\delta)
\leq E'' n^{b/(2\beta)} A(n)^{1/\beta}   (\log(\log(1/\delta)))^{96}\log(1/\delta)
\end{align*}

Therefore, by Lemmas~\ref{lem:xi1} and~\ref{lem:xi2}, the previous equation implies in the three cases for some constants $E_4, E'''$:
\begin{itemize}
\item Case $\beta<2$: With probability larger than $1  - F_4\log(1/\delta)^2 \delta$, we have
\begin{align*}
r_n &\leq E''' n^{-1/2}  (\log(\log(1/\delta)))^{96}\log(1/\delta),
\end{align*}
hence with probability larger than $1-\delta$,
\begin{align*}
r_n &\leq E_4 n^{-1/2}  \log(1/\delta) (\log(\log(1/\delta)))^{96} \sim n^{-1/2}.
\end{align*}

\item Case $\beta> 2$: With probability larger than $1 - F_4\log(1/\delta)^2  \log(n)^3 \delta$, we have
\begin{align*}
r_n &\leq E''' (n\log(n))^{-1/\beta}  (\log(\log(1/\delta)))^{96}\log(1/\delta),
\end{align*}
hence with probability larger than $1-\delta$,
\begin{align*}
r_n &\leq E_4 (n\log(n))^{-1/\beta}  (\log(\log(\log(n)/\delta)))^{96}\log(\log(n)/\delta) \sim (n\log(n))^{-1/\beta} \log\log(n) \log\log\log(n)^{96}.
\end{align*}

\item Case $\beta = 2$: With probability larger than $1  - F_4\log(1/\delta)^2  \log(n)^3 \delta$, we have
\begin{align*}
r_n &\leq E''' \log(n) n^{-1/2}  (\log(\log(1/\delta)))^{96}\log(1/\delta),
\end{align*}
hence with probability larger than $1-\delta$,
\begin{align*}
r_n &\leq E_4 \log(n) n^{-1/2}  (\log(\log(\log(n)/\delta)))^{96}\log(\log(n)/\delta) \sim \log(n) n^{-1/2} \log\log(n) \log\log\log(n)^{96}.
\end{align*}
\end{itemize}

\section{Full proof of Theorem~\ref{thm:lb}} \label{proof:lb}

%\begin{theorem}[Lower bounds]
%Let Assumption~\ref{ass:reg} be satisfied. Let $\mathcal A$ be the class of adaptive algorithms. Let
%$$b = \max(2, \beta).$$
%For any strategy in $\mathcal A$, we have that its simple regret is lower bound, with probability larger than $1/2 + u$, by
%$$c(u) n^{-1/b},$$
%where $u>0$ is a small constant.
%\end{theorem}

\subsection{Case $\beta <2$}

By Assumption~\ref{ass:reg} (equivalent to Assumption~\ref{ass:reg2}), we know that
$$E' u^{1/\beta} \geq G(u)\geq E u^{1/\beta}.$$
Assume that when pulling an arm from the reservoir, its distribution is Gaussian of mean following the distribution associated to $G$ and has variance $1$. 
Since the budget is bounded by $n$, an algorithm pulls at most $n$ arms from the arm reservoir. %Let us write $N = $ for the total number of arms pulled from the reservoir at the end of the algorithm.
%Let 
%$$I_0 = [\mu^* - E \Big(\frac{c_0}{N}\Big)^{1/\beta}, \mu^*].$$
%where $c_0$ is a constant defined in function of $\delta>0$ such that, if we write $N_0$ for the number of arms in $I_0$, we have
%$$\mathbb P_1(N_0 = 0) \geq (1 - \frac{c_0}{N})^N \geq \exp(-c_0/2) \geq 1 - \delta.$$
%So whenever $N \leq C n^{\beta/2}$ where $C>0$, there are no arms in $I_0$ with probability larger than $1-\delta$, and therefore, the regret of the algorithm is larger than 
%$$ E \Big(\frac{c_0}{N}\Big)^{1/\beta} \geq E C^{-1/\beta} c_0^{1/\beta}n^{-1/2},$$
%with probability larger than $1-\delta$.
%Assume now that $N := C_nn^{\beta/2}$ where $C_n\geq1$.
Let us define
$$I_1 = \left[\bar \mu^* - \frac{E'c_1^{1/\beta}}{\sqrt{n}} , \bar \mu^* - \frac{E(c_1')^{1/\beta}}{\sqrt{n}}\right]
\quad\mathrm{and}\quad
I_2 = \left[\bar \mu^* - \frac{E(c_1')^{1/\beta}}{\sqrt{n}} , \bar \mu^*\right],$$
where $c_1, c_1'$ are constants. If we denote $N_1$ the number of arms in $I_1$ and $N_2$ the number of arms in $I_2$ among the $n$ first arms pulled from the arm reservoir, we can use Bernstein's inequality and for $n \geq 1$ larger than a large enough constant %since $N \geq n^{\beta/2}$
$$\mathbb P_1\left(N_2 \geq n\frac{E'/Ec_1'}{n^{\beta/2}}\left(1 + \log\left(1/\delta\right)\right)\right) \leq \delta
\quad\mathrm{and}\quad
\mathbb P_1\left(N_1 \leq  n\frac{c_1 - c_1'}{n^{\beta/2}} \left(1 - \log\left(1/\delta\right)\right) \right) \leq \delta.$$
Consequently, for $c_1$ large enough when compared to $c'_1$, it implies that with probability larger than $1-2\delta$, we have that $N_1>1$ and $N_1 > N_2$. Consider the event $\xi$ of probability $1-2\delta$ where this is satisfied.

On $\xi$, a problem that is strictly easier than the initial problem is the one where an oracle points out two arms to the learner, the best arm in $I_2$ and the worst arm in $I_1$, and where the objective is to distinguish between these two arms and output the arm in $I_2$. Indeed, this problem is on $\xi$ strictly easier than an \emph{intermediate problem} where an oracle provides the set of arms in $I_1 \cup I_2$ and asks for an arm in $I_2$, since $N_1>N_2$. On $\xi$, this intermediate problem is  in turn strictly easier  than the original problem of outputting an arm in $I_2$ without oracle knowledge. Therefore, for the purpose of constructing the lower bound, we will now turn to the strictly easier problem of deciding between the arm $k^*$ with highest mean in $I_2$, and the arm $k$ with lowest mean in $I_1$ and prove the lower bound for this strictly easier problem.  %On this event, at least one arm is in $I_1$, and there is no arms in $I_0$.

Since the number of pulls on both $k^*$ and $k$ is bounded by $n$, we use the chain rule and the fact that the distributions are Gaussian to get on $\xi$
$${\rm KL}(k, k^*) \leq n (\mu - \mu^*)^2,$$
where $\mu$ is the mean of $k$ and $\mu^*$ is the mean of $k^*$.
Given $\xi$, let $p$ be the probability that $k$ is selected as the best arm, and $p^*$  the probability that $k^*$ is selected as the best arm. By Pinsker's inequality, we know that on $\xi$
$$\left|p - p^*\right| \leq \sqrt{{\rm KL}(k_1, k^*)} \leq \sqrt{n} |\mu_1 - \bar \mu^*| \leq \sqrt{n} E'\frac{c_1^{1/\beta}}{\sqrt{n}} \leq E'c_1^{1/\beta}.$$
Since there are only two arms in this simplified game, we know that on $\xi$
$$p^* \leq 1/2 + E'c_1^{1/\beta} \leq 7/12.$$
for $c_1$ small enough. By definition of $\xi$ and since the problem we considered is easier than the initial problem, we know that for all algorithms, the probability $P^*$ of selecting an arm in $I_2$ is bounded as follows where we add the probability that $\xi$ does not hold,
$$P^* \leq 7/12 + 2\delta \leq 2/3,$$
for $\delta$ small enough. This concludes the proof by definition of $I_2$.
%So this implies that the probability of pi

% which concludes the proof.

\subsection{Case $\beta \geq 2$}

By Assumption~\ref{ass:reg} (equivalent to Assumption~\ref{ass:reg2}), we know that
$$E' u^{1/\beta} \geq G(u)\geq E u^{1/\beta}.$$
Assume that when pulling an arm from the reservoir, its distribution is Gaussian of mean following the distribution associated to $G$ and has a variance $1$. 
%Let us write $N$ for the total number of arms pulled from the reservoir at the end of the algorithm.
The total number of arms pulled in the reservoir is smaller than $n$ since the budget is bounded by $n$. Let 
$$I_0 = \left[\bar \mu^* - E \left(\frac{c_0}{n}\right)^{1/\beta}, \bar \mu^*\right].$$
where $c_0$ is a constant defined in function of $\delta>0$ such that, if we denote $N_0$ for the number of arms in $I_0$, we have
$$\mathbb P_1\left(N_0 = 0\right) \geq \left(1 - \frac{c_0}{n}\right)^n \geq \exp\left(-c_0/2\right) \geq 1 - \delta.$$
Thereupon, there are no arms in $I_0$ with probability larger than $1-\delta$, and therefore, with probability larger than $1-\delta$ the regret of the algorithm is larger than 
$$ E \left(\frac{c_0}{n}\right)^{1/\beta}.$$
 %Since $\beta \geq 

% $$\tilde n_u =  \log_2\big(\frac{D\log(\max(1,2^{2\bar t_{\beta}/b}G(2^{-(u+1)})^2)/\delta)}{\max(2^{-2\bar t_{\beta}/b},G(2^{-(u+1)})^2)} \big).$$

% \todom[inline]{
% check if the $\underline{\beta}$ update was done right 
%  in the code for beta estimation we ignore beta part
% }

\section{Proof of Lemma~\ref{lem:beta}}\label{s:proofbeta}
By a union bound, we know that with probability larger than $1-\delta$, for all $k \leq N$, we have
$$\left|\widehat m_k - m_k\right| \leq c\sqrt{\frac{\log\left(N/\delta\right)}{N}}.$$

Note that by Assumption~\ref{ass:reg}, we have that with probability larger than $1-\delta$,
$$\left|\widehat m^* - \bar \mu^*\right| \leq c\left(\frac{1}{\delta N}\right)^{1/\beta}.$$
Let us write
$$v_N = c\sqrt{\frac{\log\left(N/\delta\right)}{N}} + c\left(\frac{1}{\delta N}\right)^{1/\beta}.$$

Note first that with probability larger than $1-\delta$ on the samples (not on $m_k$)
$$\frac{1}{N} \sum_{k=1}^N \mathbf 1\{\bar \mu^* -  m_k \leq N^{-\varepsilon} + v_N \} \geq \widehat p  \geq \frac{1}{N} \sum_{k=1}^N \mathbf 1\{\bar \mu^* -  m_k  \leq N^{-\varepsilon} - v_N\},$$

We now define for $l \in \{0,1\}$ 
$$p^{+} = \mathbb P_{m \sim \mathcal L} \left(\bar \mu^* - m \leq N^{-\varepsilon}+v_N\right)
\quad\mathrm{and}\quad
p^{-} = \mathbb P_{m \sim \mathcal L} \left(\bar \mu^* - m \leq N^{-\varepsilon}-v_N\right).$$
Notice that for $n$ larger than a constant that depends on $\tilde B$ of Assumption~\ref{ass:reg}, we have by Assumption~\ref{ass:reg} the following bound for $* \in \{+,-\}$, since $(v_N N^{\varepsilon})\leq 1/2\delta^{-1/\beta}$ as $\varepsilon < \min(\beta, 1/2)$,  and also for $N$ larger than a constant that depends on $\tilde B$ only
\begin{align*}
\left| - \frac{\log(p^{*})}{\varepsilon \log N} - \beta\right| &\leq \frac{ (v_N N^{\varepsilon})^\beta /\beta + \max(1,\log(\tilde E'),|\log(\tilde E)|)}{\varepsilon \log N}
\leq \frac{\delta^{-1/\beta}/\beta + \max(1,\log(\tilde E'),|\log(\tilde E)|)}{\varepsilon \log N},
\end{align*}
which implies that
$$p^* \geq c'N^{-\beta \varepsilon},$$
where $c'>1/2$ is a small constant that is larger than $\tilde E/2$ for $n$ larger than a constant that depends only on $\tilde B$.

By Hoeffding's inequality applied to Bernoulli random variables, we have that with probability larger than $1-\delta$
$$\left|\frac{1}{N} \sum_{k=1}^N \mathbf 1\{\bar \mu^* -  m_k \leq N^{-\varepsilon} + v_N \} - p^{+}\right| \leq c\sqrt{\frac{\log(1/\delta)}{N}}\eqdef w_N,$$
and the same for $p^{-}$ with $\frac{1}{N} \sum_{k=1}^N \mathbf 1\{\bar \mu^* -  m_k  \leq N^{-\varepsilon} - v_N\}$. All of this implies that with probability larger than $1 - 6 \delta$
$$p^{+} + w_N \geq \widehat p \geq p^{-} - w_N,$$
which implies that with probability larger than $1 - 6 \delta$
$$- \frac{ \log(p^{+} + w_N)}{\varepsilon \log N} \leq \widehat p \leq - \frac{\log(p^{-} - w_N) }{\varepsilon \log N},$$
i.e.,~with probability larger than $1 - 6 \delta$, since $ w_N/p^{-} \leq 1/2\sqrt{\log(1/\delta)}$ as $n$ is large enough (larger than a constant) and since $\beta \leq 1/(2\varepsilon) $
$$- \frac{ \log(p^{+})}{\varepsilon \log N} -  \frac{ 2 w_N}{p^{+}\log(n)\varepsilon} \leq - \frac{ \log(\widehat p)}{\varepsilon \log N} \geq - \frac{\log(p^{-}) }{\varepsilon \log N} +   \frac{2 w_N}{p^{-}\varepsilon \log N},$$
which implies the final claim
$$ \left|- \frac{ \log(\widehat p)}{\varepsilon \log N} - \beta\right| \leq \frac{\delta^{-1/\beta}/\beta + \sqrt{\log(1/\delta)} + \max(1,\log(\tilde E'),|\log(\tilde E)|)}{\varepsilon \log N}.$$

\end{document}